%% file: main.tex
\documentclass[twoside,11pt]{article}

\usepackage{times}
\usepackage{listings}
\usepackage{lscape}
\usepackage[abbrvbib,preprint]{jmlr2e}

\input{./command}
\usepackage{graphicx}
\usepackage{subcaption}
\usepackage{tabularx}
\usepackage{amsmath}
\newcommand{\CIRCLE}{\ensuremath{\bullet}}\newcommand{\Circle}{\ensuremath{\circ}}
\newcommand{\UCB}{\mathrm{UCB}}

\usepackage{geometry}
\usepackage{lastpage}
\jmlrheading{1}{2024}{1-\pageref    {LastPage}}{4/00}{10/00}{Lanjihong Ma*, Zhen-Yu Zhang*, Masashi Sugiyama and Zhi-Hua Zhou}{}

\ShortHeadings{Selective Ensemble Based on Preference-Directed Multi-Objective Bandits}{Ma, Zhang, Sugiyama, and Zhou}
\firstpageno{1}

\begin{document}

\title{Selective Ensemble Based on Preference-Directed Multi-Objective Bandits}

\author{\name Lanjihong Ma$^*$ \email malanjihong@zjgsu.edu.cn \\
       \addr Department of Computer Science, Zhejiang Gongshang University, China
       \AND
       \name Zhen-Yu Zhang$^*$ \email zhen-yu.zhang@riken.jp \\
       \addr Center for Advanced Intelligence Project, RIKEN, Japan
       \AND
       \name Masashi Sugiyama \email sugi@k.u-tokyo.ac.jp \\
       \addr Center for Advanced Intelligence Project, RIKEN, Japan\\
       Graduate School of Frontier Sciences, The University of Tokyo, Japan
       \AND
       \name Zhi-Hua Zhou \email zhouzh@lamda.nju.edu.cn \\
       \addr National Key Laboratory for Novel Software Technology, Nanjing University, China\\
       School of Artificial Intelligence, Nanjing University, China}

\footnotetext[1]{$^*$Lanjihong Ma and Zhen-Yu Zhang contributed equally to this work.}

\editor{}

\maketitle

\begin{abstract}
Selective ensemble for modern machine learning systems requires choosing promising model candidates under limited evaluation budgets, while downstream tasks often specify only partial preferences over capabilities such as accuracy, robustness, and reasoning. This setting naturally gives rise to a sequential decision problem under partially specified linear preferences. We formalize it as \emph{preference-directed multi-objective bandits}~{(PDMOB)}, where admissible trade-offs are represented by a polyhedral preference cone. Based on this formulation, we introduce \emph{Pareto $C$-optimality}, which recovers standard Pareto optimality and single-weight scalarization as special cases. We then propose the \emph{preference-directed upper confidence bound}~{(PrefUCB)} algorithm, which maintains directional confidence intervals to guide exploration. We analyze both indicator-based and gap-weighted regret, and establish instance-dependent logarithmic bounds for both criteria, recovering the optimal logarithmic dependence on the horizon $T$ in classical special cases. Experiments on large pre-trained model selective ensemble tasks and online asset allocation under institutional mandates validate the efficacy of our method.
\end{abstract}


\input{./sections/introduction.tex}
\input{./sections/formulation.tex}
\input{./sections/approach.tex}
\input{./sections/experiments.tex}
\input{./sections/related.tex}

\section{Conclusion}
\label{sec:conclusion}

We formulate the preference-directed multi-objective bandit problem by modeling partially specified linear preferences via a polyhedral cone \(C\), and introduce Pareto \(C\)-optimality, which unifies standard Pareto optimality and scalarization as special cases. We propose PrefUCB, an optimistic algorithm that maintains directional confidence bounds aligned with the preference cone, and establish instance-dependent logarithmic regret bounds under both indicator-based and gap-weighted criteria, recovering the optimal \(\log T\) rate in classical settings. Experiments on selective ensemble and portfolio allocation further demonstrate the effectiveness of preference-directed exploration. While these results capture the essential difficulty of the problem, our analysis incurs a linear dependence on the number of preference directions \(M\) due to a union bound across directions; this dependence may be loose when the directions are highly correlated, where the effective dimensionality of the cone is smaller than \(M\). Obtaining a tighter characterization that replaces \(M\) with a geometry-dependent quantity, together with matching lower bounds, remains open.

\bibliography{./ref.bib}



\newpage

\appendix
\input{./sections/appendix.tex}



\end{document}

%% file: command.tex
\usepackage{amsmath}
\usepackage{amssymb}
\usepackage{mathtools}
\usepackage{dsfont}
\usepackage{lipsum}
\usepackage{bm,bbm}
\usepackage{bbding}
\usepackage{paralist}
\usepackage{xspace}
\usepackage[utf8]{inputenc} 
\usepackage[T1]{fontenc}    
\usepackage{hyperref}       
\usepackage{url}            
\usepackage{amsfonts}       
\usepackage{nicefrac}       
\usepackage{microtype} 
\usepackage{multirow}
\usepackage{graphicx}
\usepackage{ragged2e}
\usepackage[most]{tcolorbox}
\definecolor{darkerlogocolor}{RGB}{20, 0, 145}
\newtcolorbox{ttcolorbox}[1][]{colframe=darkerlogocolor, colback=darkerlogocolor!4!white, title=#1}
\usepackage[font=small,labelfont=bf]{caption}
\usepackage{subfloat}
\usepackage{float}
\usepackage{algorithm,algorithmic}
\usepackage{wrapfig}
\usepackage{booktabs}
\usepackage{verbatim}
\usepackage{setspace}
\usepackage{tikz}
\usepackage{tcolorbox}
\definecolor{mydarkblue}{rgb}{0,0.08,0.45}
\hypersetup{
    colorlinks,
    breaklinks,
    linkcolor = mydarkblue,
    citecolor = mydarkblue,
    urlcolor  = black,
} 

\renewcommand{\algorithmicrequire}{ \textbf{Input:}}      
\renewcommand{\algorithmicensure}{ \textbf{Output:}}     

\renewcommand{\hat}{\widehat}

\usepackage{mathtools}
\usepackage{appendix}

\newtheorem{myThm}{Theorem}

\newtheorem{myLemma}{Lemma}

\newtheorem{myCor}{Corollary}

\newtheorem{myAssum}{Assumption}

\newtheorem{myDef}{Definition}

\newtheorem{myRemark}{Remark}

\renewcommand{\hat}{\widehat}

\def \epsilon {\varepsilon}

\definecolor{DSgray}{cmyk}{0,1,0,0}

\usepackage[textsize=tiny]{todonotes}
\usepackage{prettyref}

\newcommand{\savehyperref}[2]{\texorpdfstring{\hyperref[#1]{#2}}{#2}}
\newrefformat{eq}{\savehyperref{#1}{Eq.~(\ref*{#1})}}
\newrefformat{eqn}{\savehyperref{#1}{Equation~\ref*{#1}}}
\newrefformat{lem}{\savehyperref{#1}{Lemma~\ref*{#1}}}
\newrefformat{lemma}{\savehyperref{#1}{Lemma~\ref*{#1}}}
\newrefformat{def}{\savehyperref{#1}{Definition~\ref*{#1}}}
\newrefformat{line}{\savehyperref{#1}{Line~\ref*{#1}}}
\newrefformat{thm}{\savehyperref{#1}{Theorem~\ref*{#1}}}
\newrefformat{corr}{\savehyperref{#1}{Corollary~\ref*{#1}}}
\newrefformat{cor}{\savehyperref{#1}{Corollary~\ref*{#1}}}
\newrefformat{sec}{\savehyperref{#1}{Section~\ref*{#1}}}
\newrefformat{app}{\savehyperref{#1}{Appendix~\ref*{#1}}}
\newrefformat{appendix}{\savehyperref{#1}{Appendix~\ref*{#1}}}
\newrefformat{assum}{\savehyperref{#1}{Assumption~\ref*{#1}}}
\newrefformat{ex}{\savehyperref{#1}{Example~\ref*{#1}}}
\newrefformat{fig}{\savehyperref{#1}{Figure~\ref*{#1}}}
\newrefformat{alg}{\savehyperref{#1}{Algorithm~\ref*{#1}}}
\newrefformat{rem}{\savehyperref{#1}{Remark~\ref*{#1}}}
\newrefformat{conj}{\savehyperref{#1}{Conjecture~\ref*{#1}}}
\newrefformat{prop}{\savehyperref{#1}{Proposition~\ref*{#1}}}
\newrefformat{proto}{\savehyperref{#1}{Protocol~\ref*{#1}}}
\newrefformat{prob}{\savehyperref{#1}{Problem~\ref*{#1}}}
\newrefformat{claim}{\savehyperref{#1}{Claim~\ref*{#1}}}
\newrefformat{que}{\savehyperref{#1}{Question~\ref*{#1}}}
\newrefformat{op}{\savehyperref{#1}{Open Problem~\ref*{#1}}}
\newrefformat{fn}{\savehyperref{#1}{Footnote~\ref*{#1}}}
\allowdisplaybreaks

%% file: sections/introduction.tex

\section{Introduction}
\label{sec:introduction}

Machine learning models have achieved remarkable success across a wide range of real-world tasks~\citep{Goodfellow-et-al-2016,brown2020language,zhou2021machine}, yet their performance differs across abilities such as predictive accuracy, worst-case robustness, reasoning, and memorization~\citep{roughgarden2021beyond,jelassi2024mixture}. The choice of model therefore depends on which abilities are most relevant to the downstream application. For example, in safety-critical applications such as autonomous driving and medical diagnosis, both predictive accuracy and worst-case robustness should be considered, as a single failure under extreme conditions could lead to catastrophic consequences~\citep{wang2021advsim,robey2022probabilistically}.

\begin{figure*}[!t]
\centering
\begin{minipage}{.98\linewidth}
\centering
\includegraphics[width=0.98\textwidth]{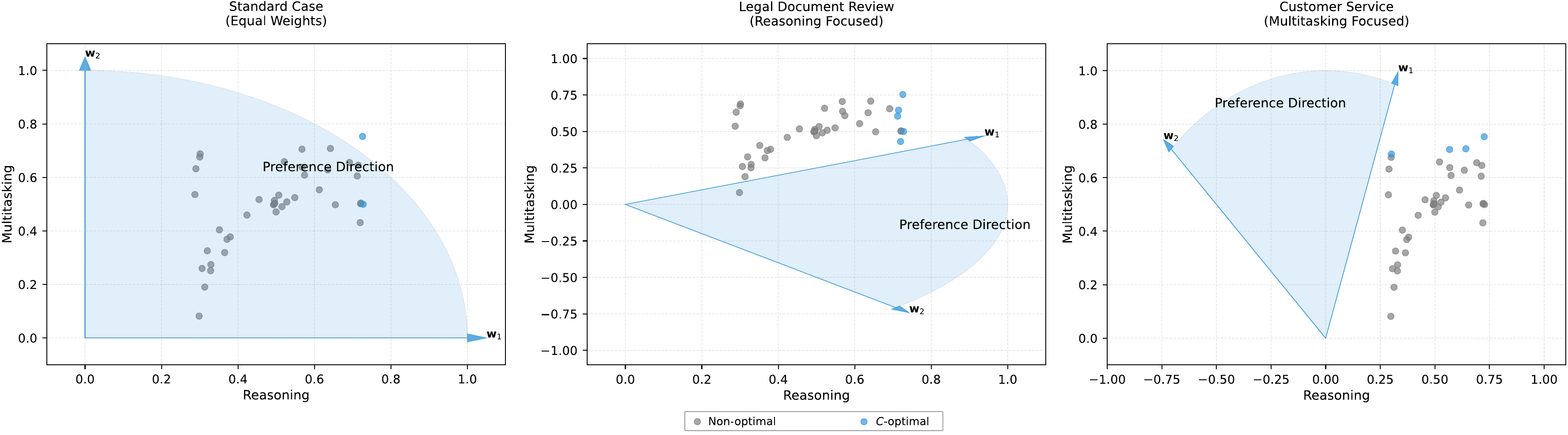}
\end{minipage}
\caption{\textbf{Visualization of preference-directed Pareto $C$-optimality in three cases using LLM performance data.} Each point is an LLM's normalized reasoning (x-axis) and multitasking (y-axis) score; negative values indicate relaxable trade-offs. The blue region shows the preference cone $C$, defining admissible trade-offs between objectives. Gray and blue points denote non-optimal and Pareto $C$-optimal LLMs, highlighting how $C$ generalizes standard Pareto optimality by incorporating user-specified trade-off structures.}
\label{fig:pd-ill}
\vspace{-2mm}
\end{figure*}

The diverse capabilities of machine learning models imply that \emph{users may prioritize different model capabilities for different downstream tasks}. Crucially, such preferences are often \emph{partially specified}: users can usually describe acceptable trade-offs, such as giving up some multi-tasking ability for stronger reasoning in legal document review tasks~\citep{DBLP:conf/bigdataconf/KeelingGGHZ22}, without specifying a precise scalar utility function. This form of preference falls into an important but overlooked gap between standard Pareto optimality~\citep{censor1977pareto}, which treats all objectives as equally non-substitutable, and scalarization~\citep{miettinen1999nonlinear}, which requires exact weights over capabilities that users often cannot provide in practice.

A natural way to address such preferences is the \emph{selective ensemble} paradigm~\citep{DBLP:journals/ai/ZhouWT02,zhou2012ensemble}, which selects from a pool of base models to explore diverse capability profiles without retraining for each task. Since evaluating every base model on the full downstream dataset is computationally expensive, we cast the selection process as a \emph{multi-objective bandit}~{(MOB)} problem~\citep{DBLP:conf/ijcnn/DruganN13}, in which only the selected models are queried on a subset of the data and suitable candidates are identified through exploration-exploitation trade-offs. However, existing MOB algorithms fall into two extremes: Pareto-based methods do not exploit user input on acceptable trade-offs, while scalarized methods require a fully specified utility that users typically cannot provide. As a result, both approaches fail to capture the partial preferences that arise in practice.

In this paper, we formulate selective ensemble with partial preferences as the \emph{preference-directed multi-objective bandit}~(PDMOB) problem. Specifically, we formalize the preference direction as a polyhedral set that represents trade-offs between different model capabilities. Based on this formalization, we extend the standard Pareto optimality to the \emph{preference-directed Pareto optimality}, which generalizes the traditional Pareto optimality by incorporating directional trade-offs between objectives. As shown in Figure~\ref{fig:pd-ill}, standard Pareto optimality uses a 90-degree polyhedral cone that treats both objectives equally, whereas our preference-directed formulation allows varying angular regions that reflect relative importance and acceptable trade-offs across scenarios.

Based on the extended Pareto optimality, the PDMOB problem involves a learner in sequentially evaluating models on subsets of the dataset while facing two central challenges: meeting the extended optimality condition that requires joint evaluation across preference directions, and managing a modified exploration-exploitation trade-off to account for directional uncertainties. To handle these challenges, we propose a novel \emph{Preference-directed Upper Confidence Bound}~(PrefUCB) algorithm, which maintains directional confidence intervals and balances exploration and exploitation along preference directions. Theoretically, we prove that PrefUCB achieves an optimal instance-dependent regret bound that matches the lower bounds of classical multi-armed bandits and standard multi-objective bandits in their respective special cases. Experiments on model selection tasks validate the efficacy of our approach. We summarize the main contributions as follows:
\begin{compactitem}[$\bullet$]
	\item We formulate the \emph{preference-directed multi-objective bandit}~(PDMOB) problem, a stochastic bandit framework in which partially specified preferences are encoded as a polyhedral cone $C$, and introduce the notion of Pareto $C$-optimality, which subsumes standard Pareto optimality and full scalarization as special cases.
	\item We propose \emph{PrefUCB}, an optimistic algorithm for learning under cone-induced preferences, and establish instance-dependent logarithmic regret bounds under two complementary criteria, recovering the classical optimal rates in the Pareto and scalarization limits.
    \item We test PrefUCB on several selective ensemble tasks, showing that preference-directed exploration more efficiently identifies candidates aligned with task-specific trade-offs.
\end{compactitem}

%% file: sections/formulation.tex

\section{Problem Formulation}
\label{sec:problem-formulation}

In this section, we first introduce preference directions and Pareto \(C\)-optimality, then formalize the PDMOB problem with its performance measures, and finally discuss the core challenges.

\subsection{Preference Direction and Pareto \(C\)-Optimality}

In this part, we first introduce the definitions of preference directions and Pareto \(C\)-optimality.

\begin{myDef}[Preference Direction]
\label{def:PD}
Given $\mathbf{W} \in \mathbb{R}^{M \times d}$, where each row encodes a linear preference over the $d$ sub-objectives, the preference direction is the closed polyhedral cone
\[ C := \{ \mathbf{x} \in \mathbb{R}^d \mid \mathbf{W} \mathbf{x} \geq \bm{0} \}. \]
\end{myDef}

\begin{myDef}[$C$-Dominance and Pareto $C$-Optimality]
\label{def:C-dom}
For $\bm{x}, \bm{y} \in \mathbb{R}^d$, we say $\bm{x}$ $C$-dominates $\bm{y}$, written $\bm{x} \succeq_C \bm{y}$, if $\bm{x} - \bm{y} \in C$. For actions $i,j$, we abbreviate $\bm{r}(i)\succeq_C\bm{r}(j)$ as $i\succeq_C j$. An action is \emph{Pareto $C$-optimal} if no other action $C$-dominates it, and we let
\[ P^*_{C} := \{ i \in [K] \mid \nexists\, j \in [K]\setminus\{i\}:\ j \succeq_{C} i \}, \]
where $[K]:=\{1,\dots,K\}$. When \(C\) is not pointed, \(\succeq_C\) is a cone-induced preorder.
\end{myDef}

The following example shows how different choices of \(C\) recover standard multi-objective formulations and express partial trade-off information.

\begin{example}
\label{ex:special-cases}
The preference cone $C$ generalizes classical settings~\citep{censor1977pareto}:
(i) taking $\mathbf{W}=\mathbf{I}_d$ gives $C=\mathbb{R}_+^d$, recovering standard Pareto optimality;
(ii) a single row $\mathbf{W}=\mathbf{w}^\top$ ($M=1$) reduces $C$ to a half-space, recovering linear scalarization;
(iii) with $M>1$ general rows, $C$ is the intersection of half-spaces, modeling partially specified linear preferences built from user trade-offs (e.g., ``favor sub-objective $1$ over $2$'' corresponds to a row $[1,-1,0,\dots,0]$).
\end{example}

Before stating the learning objective and regret bounds, we introduce normalized projections along each preference direction and the corresponding preference gaps.

\begin{myDef}[Normalized Preference Projections]
\label{def:norm-proj}
For direction $k\in[M]$, the normalized projection of $\bm{r}\in\mathbb{R}^d$ is
\[
S_k(\bm{r}) := \frac{\mathbf{w}_k^\top\bm{r} - L_k}{U_k - L_k},
\quad
L_k \!=\! \sum_{i}\min(0,\mathbf{w}_{k,i}),\ U_k \!=\! \sum_{i}\max(0,\mathbf{w}_{k,i}),
\]
so that $S_k(\bm{r})\in[0,1]$ whenever $\bm{r}\in[0,1]^d$.
\end{myDef}

\begin{myAssum}[Strict Separability]
\label{assump:separability}
Every non-Pareto-$C$-optimal action $a\notin P_C^*$ admits some Pareto-$C$-optimal dominator $a^*\in\mathcal D(a)$ with
\[
\min_{k\in[M]} \mathbf w_k^\top \bigl[\bm r(a^*)-\bm r(a)\bigr] > 0.
\]
\end{myAssum}

\begin{myRemark}
Assumption~\ref{assump:separability} plays the same role as the positive-gap condition in classical stochastic bandits: it rules out weakly dominated actions, for which instance-dependent logarithmic regret bounds in $1/\Delta_a$ are generally unattainable.
\end{myRemark}

\begin{myDef}[Preference Gaps]
\label{def:gaps}
Let $\mathcal{D}(a) = \{a^*\in P_C^* : a^* \succeq_C a,\ a^*\neq a\}$. For $a\notin P_C^*$, define
\[
\Delta_a := \max_{a^*\in\mathcal{D}(a)} \min_{k\in[M]}
\frac{\mathbf{w}_k^\top[\bm{r}(a^*) - \bm{r}(a)]}{U_k - L_k} > 0,
\]
and set $\Delta_a:= 0$ for $a\in P_C^*$.
\end{myDef}

\subsection{Preference-Directed Multi-Objective Bandits}

In this part, we formalize the PDMOB problem induced by the preference cone \(C\). We consider an action set $\mathcal A := [K]$, where each $a\in[K]$ has an unknown reward vector $\bm r(a)\in[0,1]^d$. At round $t\in[T]$, the learner selects $a_t$ and observes a noisy reward $\bm y_t\in[0,1]^d$ with $\mathbb E[\bm y_t\mid a_t]=\bm r(a_t)$. Since Pareto $C$-optimality of $a_t$ is not directly observable, the goal is to play Pareto $C$-optimal actions as often as possible.

We evaluate performance via two complementary measures: the \emph{indicator regret}, capturing \emph{whether} a chosen action is Pareto $C$-optimal, and the \emph{gap-weighted regret}, capturing \emph{how severely} a suboptimal choice violates preference-directed optimality.

\begin{figure*}[!t]
\begin{minipage}{.98\linewidth}
\begin{algorithm}[H]
\begin{flushleft}
At each $t\in[T]$:
\begin{algorithmic}[1]
\STATE The learner selects $a_t\in[K]$ and incurs unobserved regrets $\ell_t^{01}(a_t)=\mathbb{I}(a_t\notin P_C^*)$ and $\ell_t^{\mathrm{gap}}(a_t)=\Delta_{a_t}\mathbb{I}(a_t\notin P_C^*)$;
\STATE The learner observes $\bm y_t\in[0,1]^d$ with $\mathbb E[\bm y_t]=\bm r(a_t)$.
\end{algorithmic}
\textbf{Objective.} Minimize $\mathcal{R}_T^{01}=\sum_{t=1}^T \mathbb{I}(a_t\notin P_C^*)$ and $\mathcal{R}_T^{\mathrm{gap}}=\sum_{t=1}^T \Delta_{a_t}\mathbb{I}(a_t\notin P_C^*)$.
\end{flushleft}
\end{algorithm}
\end{minipage}
\caption{The Preference-Directed Multi-Objective Bandits (PDMOB) problem.}
\label{fig:learning_procedure}
\end{figure*}

\begin{myDef}[Cumulative Regrets]
\label{def:regrets}
Over $T$ rounds, the cumulative \emph{indicator} and \emph{gap-weighted} regrets are
\[
\mathcal{R}_T^{01} := \sum_{t=1}^T \mathbb{I}(a_t\notin P_C^*),
\qquad
\mathcal{R}_T^{\mathrm{gap}} := \sum_{t=1}^T \Delta_{a_t}\,\mathbb{I}(a_t\notin P_C^*).
\]
\end{myDef}

Together, these definitions yield the PDMOB protocol summarized in Figure~\ref{fig:learning_procedure}.

\subsection{Core Challenges}

In this part, we discuss the core challenges of the PDMOB problem, focusing on two structural difficulties that do not arise in standard Pareto- or scalarization-based bandits.

\begin{compactitem}[$\bullet$]
\item \textbf{Joint confidence across directions.} Standard Pareto dominance is verified coordinate-wise, so per-objective estimation errors decouple. Pareto $C$-optimality requires all $M$ polyhedral constraints to hold \emph{jointly}: an error along any single direction can misclassify a suboptimal arm, forcing simultaneous confidence control over all directions.

\item \textbf{Directional, arm-dependent gaps.} Unlike fixed coordinate-wise gaps in MOB, the effective gap $\Delta_a$ depends on both the chosen dominator $a^*$ and the tightest separating direction. Different suboptimal arms are hardest to distinguish along different directions, so uniform exploration across coordinates or arms cannot adapt to this geometry.
\end{compactitem}

%% file: sections/approach.tex

\section{Algorithm and Theory}
\label{sec:algorithm-and-theory}

In this section, we first give a directional characterization of \(C\)-dominance, turning the cone-based preference relation to scalar comparisons. This characterization yields structural requirements that guide the design of an optimistic algorithm for PDMOB. We then state the regret guarantee of the proposed algorithm.

\subsection{A Directional Characterization of \texorpdfstring{$C$}{C}-Dominance}
\label{sec:decomposition}

By Definition~\ref{def:C-dom}, action \(i\) \(C\)-dominates action \(j\), written \(i \succeq_C j\), if and only if
\[
\bm r(i)-\bm r(j)\in C.
\]
Since \(C\) is defined by the preference direction matrix $\mathbf W = [\mathbf w_1^\top,\mathbf w_2^\top,\ldots,\mathbf w_M^\top] \in \mathbb R^{M\times d}$, this is equivalent to $\mathbf W(\bm r(i)-\bm r(j)) \ge \bm 0$. Using the normalized directional projections \(S_k(\cdot)\) from Definition~\ref{def:norm-proj}, which are positive affine transformations of \(\mathbf w_k^\top \bm r\), we get the equivalent form

\begin{equation}
\label{eq:C-domination}
i \succeq_C j
\iff
\min_{k\in[M]}
\bigl[
S_k(\bm r(i)) - S_k(\bm r(j))
\bigr]
\ge 0.
\end{equation}

Eq.~\eqref{eq:C-domination} shows that \(C\)-dominance can be verified via \(M\) scalar comparisons, avoiding explicit cone geometry. 
When the reward vectors are known, Pareto-\(C\)-optimality can be checked pairwise in \(\mathcal{O}(MK^2)\) time for \(K := |\mathcal{A}|\) actions. 
More importantly, Eq.~\eqref{eq:C-domination} suggests that a sequential algorithm for the PDMOB problem should estimate rewards along preference directions and aggregate them through a worst-case criterion.

\begin{algorithm}[t]
\caption{Preference-Directed Upper Confidence Bound (PrefUCB)}
\label{alg:prefucb}
\begin{algorithmic}[1]
\STATE \textbf{Input:} action set \(\mathcal A\), preference direction matrix \(\mathbf W\), horizon \(T\)
\STATE \textbf{Initialization:} pull each action \(a\in\mathcal A\) once; set \(n(a)\gets 1\) and initialize \(\hat{\bm r}(a)\) using the observed reward vector
\FOR{\(t = |\mathcal A|+1, |\mathcal A|+2, \ldots, T\)}
    \FOR{each action \(a\in\mathcal A\)}
        \FOR{each direction \(k\in[M]\)}
            \STATE Compute \(\hat S_k(a)\) by~\eqref{eq:empirical-proj}
            \STATE Compute \(U_{a,k}(t)\) by~\eqref{eq:dir-ucb}
        \ENDFOR
        \STATE Set \(U_a(t)\gets \min_{k\in[M]} U_{a,k}(t)\)
    \ENDFOR
    \STATE Select \(a_t \in \arg\max_{a\in\mathcal A} U_a(t)\)
    \STATE Observe feedback \(\bm y_t(a_t)\in[0,1]^d\)
    \STATE Update
    \[
    n(a_t)\gets n(a_t)+1,
    \qquad
    \hat{\bm r}(a_t)
    \gets
    \frac{(n(a_t)-1)\hat{\bm r}(a_t)+\bm y_t(a_t)}{n(a_t)}
    \]
\ENDFOR
\end{algorithmic}
\end{algorithm}

\subsection{Preference-Directed Upper Confidence Bound}
\label{sec:method}

In this part, we introduce our proposed algorithm. The characterization in Eq.~\eqref{eq:C-domination} leads to three requirements on the optimistic algorithm for the PDMOB problem:
\begin{compactitem}[$\bullet$]
    \item \textbf{Directional estimation.} \(C\)-dominance is checked along directions, not coordinates. So confidence must be tracked for each \(\mathbf w_k^\top\bm r(a)\), not for each coordinate of \(\bm r(a)\). Tracking coordinates wastes effort on directions orthogonal to \(C\).
    \item \textbf{Worst-direction aggregation.} Pareto-\(C\)-optimality requires an action to compete on \emph{all} directions at once (the \(\min_k\) in Eq.~\eqref{eq:C-domination}). A worst-direction score gives a conservative single-index certificate aligned with this component-wise dominance structure. Averages or sums may accept actions that are strong on some directions but fail on others.
    \item \textbf{KL-sharp confidence.} The projections \(S_k(\bm r(a))\) yield bounded \([0,1]\) feedback, so each direction is a scalar bounded-mean problem, where KL-type bounds are instance-optimal.
\end{compactitem}
These requirements motivate the \emph{Preference-Directed Upper Confidence Bound}~{(PrefUCB)} algorithm, whose main components are summarized in Algorithm~\ref{alg:prefucb}.

\paragraph{Directional empirical rewards.}
Let \(\hat{\bm r}(a)\) be the empirical mean reward vector of action \(a\) and let \(n(a)\) be its pull count. For direction \(k\in[M]\), the directional empirical reward is
\begin{equation}
\label{eq:empirical-proj}
\hat S_k(a)
=
S_k(\hat{\bm r}(a))
=
\frac{\mathbf w_k^\top \hat{\bm r}(a)-L_k}{U_k-L_k},
\end{equation}
with \(L_k\) and \(U_k\) from Definition~\ref{def:norm-proj}. Since \(S_k(\cdot)\) is affine, \(S_k(\bm y_t(a))\) stays in \([0,1]\) and
\[
\mathbb E\!\left[S_k(\bm y_t(a))\right] = S_k(\bm r(a)).
\]
So for each direction, estimating \(S_k(\bm r(a))\) is a scalar bounded-mean problem.

\paragraph{Directional confidence index.}
Let
\[
\mathrm{KL}(x,q)
=
x\log\frac{x}{q} + (1-x)\log\frac{1-x}{1-q}
\]
be the Bernoulli KL divergence. Following the instance-optimal bound for bounded scalar means~\citep{garivier2011kl}, we set the directional confidence index for each action \(a\) and direction \(k\) at round \(t\) as
\begin{equation}
\label{eq:dir-ucb}
U_{a,k}(t) = \sup\left\{ q\in[\hat S_k(a),1]: n(a) \mathrm{KL} \bigl(\hat S_k(a),q\bigr)\le \beta_t \right\},
\end{equation}
where $\beta_t = \log t$. This is the KL-UCB confidence index applied to the scalar directional feedback \(S_k(\bm y_t(a))\). Since \(S_k(\bm y_t(a))\in[0,1]\), each preference direction becomes a one-dimensional bounded-mean bandit problem; the Bernoulli KL index above is the standard KL-UCB confidence index for bounded rewards and provides an optimistic upper confidence bound for \(S_k(\bm r(a))\).

\paragraph{Worst-direction aggregation as a $C$-dominance certificate.}
The directional indices are combined into a single score
\begin{equation*}
U_a(t) = \min_{k\in[M]} U_{a,k}(t).
\end{equation*}
The minimum is not the only possible way to use directional confidence bounds; one could also design pairwise dominance tests or maintain an optimistic estimate of the full Pareto-\(C\) frontier. We use the worst-direction score because it is a simple single-index certificate that is monotone with respect to componentwise directional comparisons: if an arm is uniformly below a Pareto-\(C\)-optimal dominator in all preference directions, then its minimum directional index is also below that dominator's minimum directional index. This property is exactly what the elimination argument in Appendix~\ref{app:theorem} exploits.

Finally, PrefUCB picks the action with the largest score:
\begin{equation}
\label{eq:selection-rule}
a_t \in \arg\max_{a\in\mathcal A} U_a(t).
\end{equation}
This rule is a selection policy rather than a Pareto-\(C\) frontier identification procedure: PrefUCB only needs to avoid non-Pareto-\(C\)-optimal actions under the regret criteria, not to enumerate all actions in \(P_C^\ast\).

The same structure explains why two natural baselines fall short. A coordinate-wise Pareto-UCB tracks confidence on each objective dimension \(i\in[d]\) and checks pairwise Pareto dominance. This spends effort on directions orthogonal to \(C\) and gives a rate that scales with \(d\) rather than \(M\), which is loose when \(M\ll d\). A scalarized UCB goes the other way and collapses the reward onto a single \(\mathbf w^\top\bm r\). It recovers only one face of the Pareto-\(C\)-optimal set and cannot check \(C\)-dominance on the other \(M-1\) directions, so it is inconsistent with Definition~\ref{def:C-dom} whenever \(M>1\). PrefUCB sits between these two extremes: directional enough to respect \(C\), and aggregated tightly enough to scale with \(M\) rather than \(d\).

For fixed \(a\) and \(k\), the set in Eq.~\eqref{eq:dir-ucb} is an interval because \(q\mapsto \mathrm{KL}(\hat S_k(a),q)\) is continuous and strictly increasing on \([\hat S_k(a),1]\). So each \(U_{a,k}(t)\) can be computed by one-dimensional binary search. The per-round cost is linear in the number of actions and directions, up to the precision of the root-finding routine.

\subsection{Theoretical Results}
\label{sec:theory}

In this part, we show the instance-dependent regret guarantee for the PrefUCB algorithm. 

\begin{myThm}
\label{thm:main}
Under Assumption~\ref{assump:separability}, for a PDMOB problem with preference direction matrix \(\mathbf W\in\mathbb R^{M\times d}\), Algorithm~\ref{alg:prefucb} satisfies
\[
\mathbb E\!\left[R_T^{01}\right]
=
\mathcal O\!\left(
\sum_{a\notin P_C^*}
\left(
\frac{M\log T}{\Delta_a^2}+1
\right)
\right),
\]
where \(\Delta_a\) is the preference gap from Definition~\ref{def:gaps}. The full proof is in Appendix~\ref{app:theorem}.
\end{myThm}

The bound matches the structural analysis of Section~\ref{sec:method}: the \(\log T\) factor comes from KL-optimal per-direction estimation, the \(M\) factor from a union bound over the \(M\) directions (not from the ambient dimension \(d\)), and the \(\Delta_a^{-2}\) factor from the uniform directional margin given by \(C\)-dominance. Each term maps to a specific part of the cone-induced structure, not to a generic UCB template.

\paragraph{Proof sketch.}
Fix any non-Pareto-\(C\)-optimal action \(a\). By the gap definition and the separability assumption, there exists a Pareto-\(C\)-optimal dominator \(a^\star\) that is uniformly better than \(a\) along every normalized preference direction, with margin \(\Delta_a\). PrefUCB estimates exactly these directional projections. Once \(a\) has been sampled enough for its directional confidence widths to be smaller than this margin, the worst-direction optimistic score of \(a\) falls below that of \(a^\star\), unless one of two bad events occurs: \(a\)'s empirical projection is overly optimistic in some direction, or the confidence index of \(a^\star\) fails to cover its true directional value. The first type of event is controlled by concentration over the sample counts of \(a\), and the second by the standard adaptive KL-UCB coverage bound. Taking a union bound over the \(M\) preference directions gives the arm-wise logarithmic pull bound, and summing over all non-Pareto-\(C\)-optimal actions yields the theorem.

\begin{myCor}
\label{cor:gap-regret}
Under the same conditions as Theorem~\ref{thm:main}, Algorithm~\ref{alg:prefucb} satisfies
\[
\mathbb E\!\left[R_T^{\mathrm{gap}}\right]
=
\mathcal O\!\left(
\sum_{a\notin P_C^*}
\left(
\frac{M\log T}{\Delta_a}+\Delta_a
\right)
\right).
\]
\end{myCor}

\begin{myRemark}
\label{rem:reduction}
Several classical settings are \emph{degenerate} cases of PDMOB, obtained by collapsing the cone. Setting \(\mathbf W=\mathbf I_d\) drops the preference information and makes \(C\) the positive orthant, which gives standard Pareto bandits~\citep{DBLP:conf/ijcnn/DruganN13}; here PrefUCB's \(M\)-dependence becomes \(d\), matching the known rate \(\mathcal O\!\left(\sum_{a\notin P_C^*} d\log T/\Delta_a^2\right)\). Setting \(M=1\) collapses \(C\) to a halfspace and removes the multi-direction structure, giving scalarized bandits~\citep{lai1985asymptotically,DBLP:journals/jmlr/Auer02} with the classical \(\mathcal O\!\left(\sum_{a:\Delta_a>0}\log T/\Delta_a^2\right)\) rate. Further setting \(d=1\) gives the classical multi-armed bandit with the same optimal bound. These reductions show that PrefUCB's structure is active only when the cone is nontrivial; when it degenerates, the algorithm and the bound reduce to the matching classical benchmark.
\end{myRemark}

\begin{myRemark}
\label{rem:scalability}
PrefUCB depends on the number of preference directions \(M\), not directly on the objective dimension \(d\). So when the preferences are sparse (\(M\ll d\)), the algorithm stays efficient in high-dimensional objective spaces.
\end{myRemark}

%% file: sections/experiments.tex

\section{Experiments}
\label{sec:experiments}

In this section, we evaluate PrefUCB on two complementary tasks: selective ensemble of LLMs with simulated preferences and online asset allocation with mandates learned from historical data. Both settings naturally fit the MOB framework where multi-dimensional rewards are observed only for selected actions.

\subsection{Selective Ensemble on LLMs}

We test PrefUCB on selective ensemble tasks where a learner sequentially selects models from a candidate pool under user-defined multi-criteria preferences. This naturally fits the MOB setting: each model is an action with multi-dimensional rewards, and due to high LLM evaluation costs, feedback is observed only for the selected model. We run experiments for \(T=10{,}000\) rounds with stochastic feedback sampled from a Beta distribution with concentration \(c=1\), reporting averages over 10 independent trials.

\textbf{Experimental setups.}~We construct our action set from the Open LLM Leaderboard\footnote{\url{https://huggingface.co/spaces/open-llm-leaderboard/open_llm_leaderboard}}, selecting $K=42$ representative models characterized by $d=7$ capabilities: accuracy, advanced problem-solving, reasoning, domain knowledge, interaction consistency, multitasking, and processing speed. We evaluate on four real-world scenarios with distinct preference structures: \emph{Legal Document Review} prioritizes reasoning while trading off multitasking; \emph{Customer Service} prioritizes multitasking while trading off problem-solving; \emph{High-frequency Trade} prioritizes speed while trading off problem-solving; and \emph{Scientific Discovery} prioritizes problem-solving while trading off speed. Details are provided in Appendix~\ref{app:preference-direction-matrices}.

\textbf{Baselines.} We compare PrefUCB against ParetoUCB~\citep{DBLP:conf/ijcnn/DruganN13}, ScalarizedUCB\(_{r.r.}\)~\citep{DBLP:conf/ijcnn/DruganN13}, and GT-PF (ground truth Pareto front). Details are provided in Appendix~\ref{app:contenders}.

\begin{figure*}[t]
  \centering
\begin{minipage}{.89\textwidth}
\centering
\includegraphics[width=0.98\textwidth]{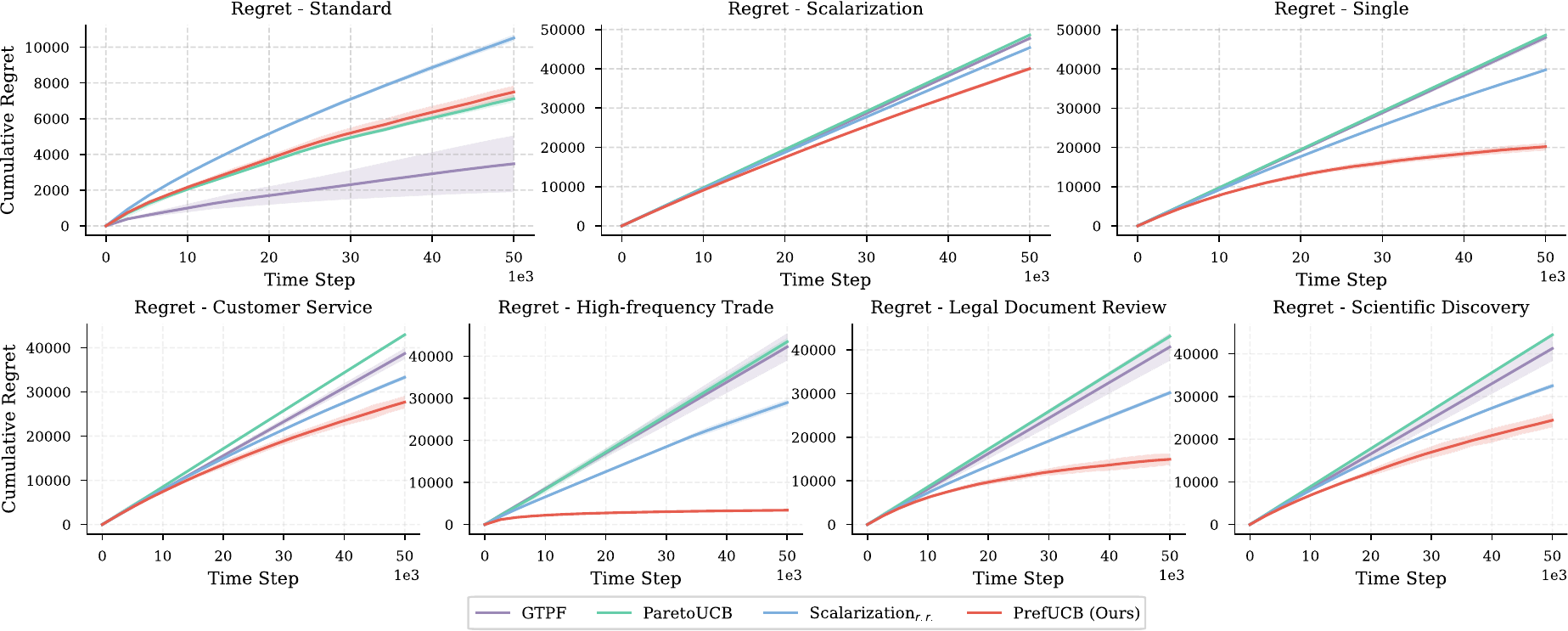}
\end{minipage}
\caption{\textbf{Cumulative regret across benchmark and real-world scenarios.} Shaded areas represent one standard deviation over repeated runs. PrefUCB consistently performs well across both synthetic and practical tasks.}
\vspace{-2mm}
\label{fig:regret}
\end{figure*}

\begin{table*}[!t]
\centering
\resizebox{.98\textwidth}{!}{
\begin{tabular}{c|l|c|c|c|c}
\toprule
\textbf{Performance} & \textbf{Method} & \textbf{Customer Service} & \textbf{Legal Document Review} & \textbf{High-frequency Trade} & \textbf{Scientific Discovery} \\
\midrule
\multirow{4}{*}{\shortstack{Main\\Objective ($\uparrow$)}}

& $\mbox{ScalarizedUCB}_\text{r.r.}$ & .590 $\pm$ .181 & .627 $\pm$ .192 & .650 $\pm$ .235 & .593 $\pm$ .171 \\
& ParetoUCB & .533 $\pm$ .190 & .522 $\pm$ .203 & .511 $\pm$ .210 & .531 $\pm$ .184 \\
& GT-PF & .573 $\pm$ .176 & .557 $\pm$ .202 & .536 $\pm$ .216 & .567 $\pm$ .172 \\
& PrefUCB & \textbf{\boldmath$.625 \pm .171$} & \textbf{\boldmath$.699 \pm .154$} & \textbf{\boldmath$.804 \pm .144$} & \textbf{\boldmath$.640 \pm .157$} \\
\midrule
\multirow{4}{*}{\shortstack{Relaxable\\Objective ($\downarrow$)}}

& $\mbox{ScalarizedUCB}_\text{r.r.}$ & .560 $\pm$ .196 & .585 $\pm$ .186 & .466 $\pm$ .209 & .463 $\pm$ .181 \\
& ParetoUCB & .518 $\pm$ .205 & .532 $\pm$ .191 & .518 $\pm$ .201 & .511 $\pm$ .209 \\
& GT-PF & .557 $\pm$ .200 & .571 $\pm$ .177 & .539 $\pm$ .201 & .539 $\pm$ .218 \\
& PrefUCB & \textbf{\boldmath$.457 \pm .196$} & \textbf{\boldmath$.460 \pm .160$} & \textbf{\boldmath$.374 \pm .166$} & \textbf{\boldmath$.436 \pm .154$} \\
\midrule
\multirow{4}{*}{\shortstack{Optimal Selection\\Ratio ($\uparrow$)}}

& $\mbox{ScalarizedUCB}_\text{r.r.}$ & .645 $\pm$ .191 & .681 $\pm$ .170 & .695 $\pm$ .167 & .646 $\pm$ .186 \\
& ParetoUCB & .571 $\pm$ .248 & .568 $\pm$ .249 & .569 $\pm$ .252 & .555 $\pm$ .257 \\
& GT-PF & .612 $\pm$ .223 & .593 $\pm$ .235 & .577 $\pm$ .244 & .587 $\pm$ .238 \\
& PrefUCB & \textbf{\boldmath$.690 \pm .157$} & \textbf{\boldmath$.801 \pm .080$} & \textbf{\boldmath$.946 \pm .014$} & \textbf{\boldmath$.722 \pm .138$} \\
\bottomrule
\end{tabular}
}
\caption{\textbf{Performance of PrefUCB on four real-world LLM selection tasks.} Each entry reports mean $\pm$ standard deviation over $10$ trials. The main objective is the prioritized capability; the relaxable objective can be traded off. The optimal selection ratio ($\uparrow$) denotes the proportion of optimal actions selected. Bold highlights best results.}
\label{tab:performance_comparison}
\vspace{-2mm}
\end{table*}

\textbf{Main results.} Since the indicator and gap-weighted regret follow the same trend, we report only the latter in Figure~\ref{fig:regret}. When $\mathbf{W}=\mathbf{I}_d$, PrefUCB matches ParetoUCB, as expected: the cone reduces to the classical Pareto order and no directional structure is available. Once preferences induce asymmetric trade-offs, the gap widens. Pareto-based methods are too coarse in this regime, since they treat actions as equivalent even when they differ along user-relevant directions; uniform sampling from the Pareto set, as in GT-PF, therefore ignores the finer structure given by the cone. PrefUCB instead exploits this structure to concentrate exploration on profiles aligned with the target trade-off. Table~\ref{tab:performance_comparison} shows that the gain extends beyond cumulative regret: PrefUCB more reliably preserves prioritized capabilities, relaxes secondary ones, and selects preference-optimal models more often.

\textbf{Trends in sub-objective performance.} Figure~\ref{fig:subobjective_trends} shows how PrefUCB adapts performance across objectives according to the specified preferences. It consistently improves preferred objectives, stabilizes moderate objectives after initial exploration, and trades off relaxable objectives when beneficial. Compared with baselines, which exhibit more uniform performance patterns across objectives, PrefUCB achieves more targeted preference-directed trade-offs. Table~\ref{tab:performance_comparison} summarizes these trade-offs quantitatively.

\begin{figure}[!t]
\centering
\begin{subfigure}{0.82\textwidth}
  \centering
  \includegraphics[width=\linewidth]{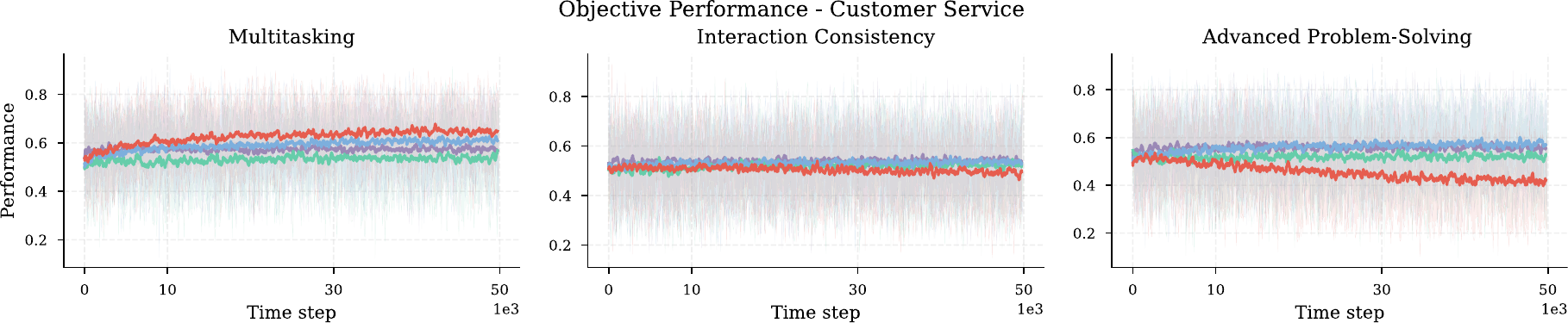}
\end{subfigure}
\vspace{0.6em}
\begin{subfigure}{0.82\textwidth}
  \centering
  \includegraphics[width=\linewidth]{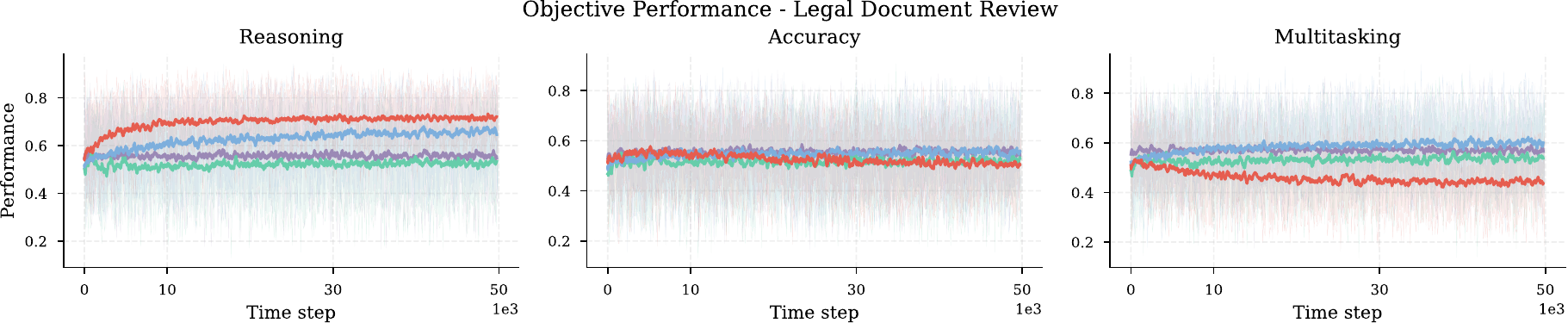}
\end{subfigure}
\vspace{0.6em}
\begin{subfigure}{0.82\textwidth}
  \centering
  \includegraphics[width=\linewidth]{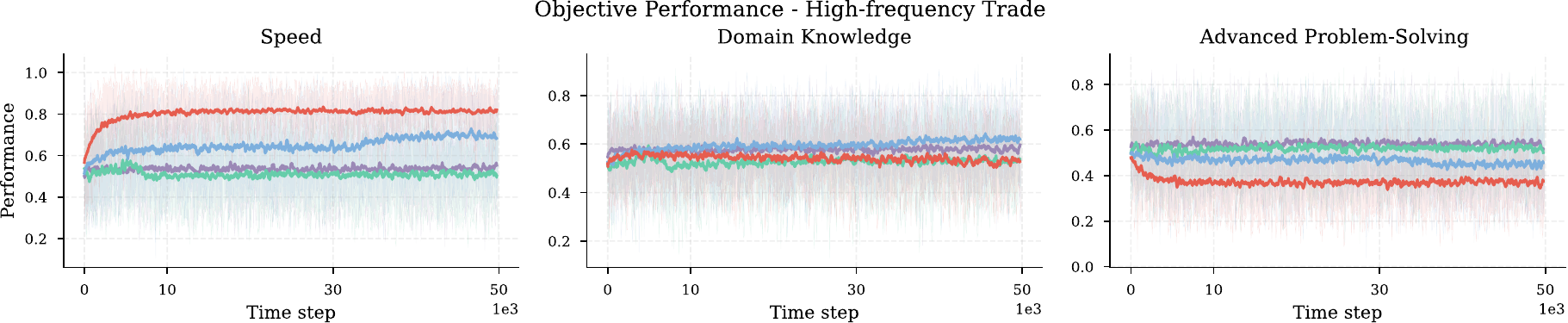}
\end{subfigure}
\vspace{0.6em}
\begin{subfigure}{0.82\textwidth}
  \centering
  \includegraphics[width=\linewidth]{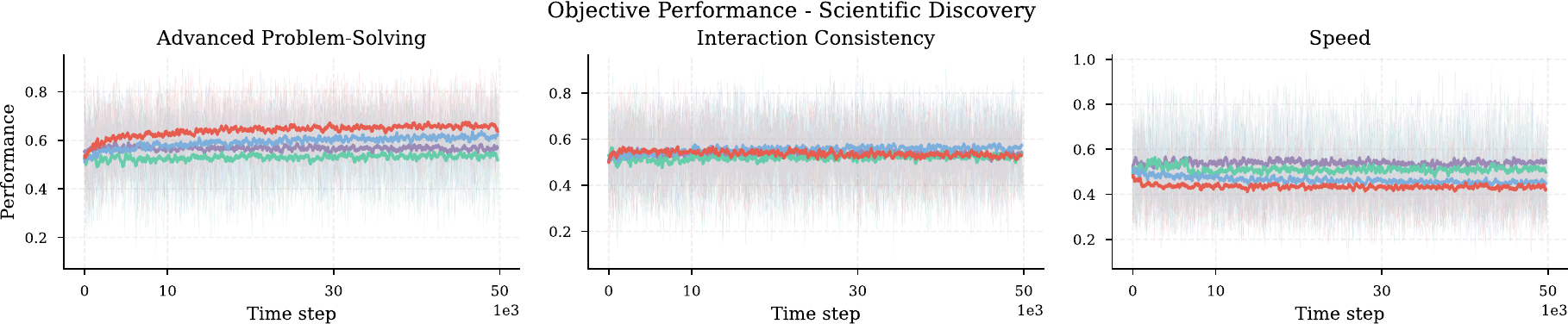}
\end{subfigure}
\vspace{0.4em}
\begin{subfigure}{0.68\textwidth}
  \centering
  \includegraphics[width=\linewidth]{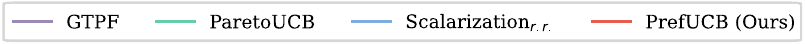}
\end{subfigure}
\caption{\textbf{Temporal trends of sub-objectives across real-world scenarios, grouped by preference level.} Solid lines denote the 100-round moving average, and shaded regions indicate the true rewards. PrefUCB consistently achieves better trade-offs by prioritizing preferred objectives, relaxing less critical ones, and maintaining moderate objectives.}
\vspace{-2mm}
\label{fig:subobjective_trends}
\end{figure}

\subsection{Online Asset Allocation}

We evaluate PrefUCB on online asset allocation, formulated as a PDMOB where the learner sequentially selects investment strategies to satisfy institutional mandates.

\textbf{Experimental setups.} We use the NYSE(O) dataset with 36 NYSE stocks over 22 years~\citep{li2014online} and construct a pool of $K=10$ base strategies ranging from passive benchmarks (Uniform Buy-and-Hold) to risk-aware approaches (Risk Parity)~\citep{cover1991universal}. At each day $t$, the learner selects a strategy and observes a 7-dimensional reward vector: Cumulative Wealth Growth (CWG), Sharpe Ratio (SR), Maximum Drawdown (MDD), Sortino Ratio, Alpha, Turnover (TO), and Calmar Ratio (CR). We run experiments for all data $T=5{,}651$ days' stock price, averaged over 20 independent trials.

Unlike the previous experiments, preference directions are now constructed according to eight real-world financial mandates that encode distinct institutional investment profiles: \emph{Aggressive Growth}, \emph{Market Neutral Hedge}, \emph{High Frequency Trading}, \emph{Income Focused}, \emph{Low Volatility}, \emph{Momentum Concentrated}, \emph{Risk Parity Enhanced}, and \emph{Balanced Growth}. Each mandate is specified by a preference direction $W \in \mathbb{R}^{3 \times 7}$ derived from historical performance data to reflect realistic institutional priorities and trade-offs. Details are provided in Appendix~\ref{app:mandate_details}.

\textbf{Baselines.} Following the same evaluation protocol as in the previous experiment, we compare PrefUCB against ParetoUCB~\citep{DBLP:conf/ijcnn/DruganN13}, ScalarizedUCB\(_{r.r.}\)~\citep{DBLP:conf/ijcnn/DruganN13}, and GT-PF. Details are provided in Appendix~\ref{app:contenders}.

\textbf{Main Results.} Since the indicator and gap-weighted regret follow the same trend, we report only the latter in Figure~\ref{fig:regret_financial} and Table~\ref{tab:main_results}. When the mandate is broadly aligned with the standard Pareto order, PrefUCB matches ParetoUCB, indicating that our formulation does not sacrifice performance in near-Pareto regimes. The gain emerges when the mandate specifies \emph{directed trade-offs} rather than uniform improvement across objectives. In this regime, Pareto-based exploration is too diffuse, wasting pulls on directions that are Pareto-valid but irrelevant to the mandate, while scalarized baselines commit to a fixed aggregation and are brittle when the trade-off is only partially specified. PrefUCB instead uses the cone to focus exploration on reward directions consistent with the mandate, which speeds up discrimination among strategies that look alike under the standard Pareto criterion but differ under directional preferences. The weaker performance of GT-PF confirms that the challenge here is not identifying the Pareto set, but resolving the finer structure the mandate induces within it. Preference-directed optimism is thus most useful when institutional requirements refine, rather than replace, the standard notion of multi-objective optimality.

\begin{figure*}[!t]
  \centering
\begin{minipage}{.98\textwidth}
\centering
\includegraphics[width=0.98\textwidth]{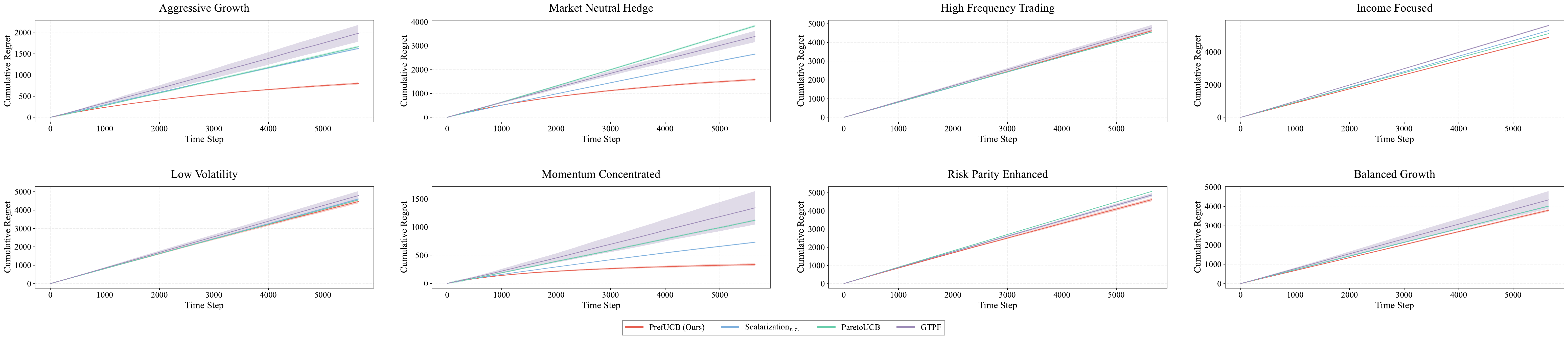}
\end{minipage}
\caption{\textbf{Cumulative gap-weighted regret analysis for online asset allocation tasks.} Shaded areas represent standard deviation. PrefUCB demonstrates robust performance across diverse financial environments.}
\label{fig:regret_financial}
\end{figure*}

\begin{table*}[!t]
\centering
\begin{minipage}{\textwidth}
\centering
\resizebox{0.98\textwidth}{!}{%
\begin{tabular}{lcccc}
\toprule
{Scenario} & {{PrefUCB}} & {$\mbox{ScalarizedUCB}_{\text{r.r.}}$} & {{ParetoUCB}} & {{GT-PF}} \\
\midrule
Aggressive Growth & \textbf{\boldmath$0.142 \pm 0.004~(1)$} & 0.288 $\pm$ 0.003~(2)~\CIRCLE & 0.295 $\pm$ 0.004~(3)~\CIRCLE & 0.351 $\pm$ 0.035~(4)~\CIRCLE \\
Market Neutral Hedge & \textbf{\boldmath$0.280 \pm 0.008~(1)$} & 0.469 $\pm$ 0.005~(2)~\CIRCLE & 0.679 $\pm$ 0.008~(4)~\CIRCLE & 0.601 $\pm$ 0.042~(3)~\CIRCLE \\
High Frequency Trading & 0.808 $\pm$ 0.004~(2) & 0.819 $\pm$ 0.009~(3)~\CIRCLE & \textbf{\boldmath$0.806 \pm 0.003~(1)$}~\Circle & 0.847 $\pm$ 0.028~(4)~\CIRCLE \\
Income Focused & \textbf{\boldmath$0.867 \pm 0.007~(1)$} & 0.941 $\pm$ 0.002~(3)~\CIRCLE & 0.909 $\pm$ 0.002~(2)~\CIRCLE & 0.997 $\pm$ 0.007~(4)~\CIRCLE \\
Low Volatility & \textbf{\boldmath$0.790 \pm 0.017~(1)$} & 0.815 $\pm$ 0.006~(3)~\CIRCLE & 0.806 $\pm$ 0.004~(2)~\CIRCLE & 0.848 $\pm$ 0.044~(4)~\CIRCLE \\
Momentum Concentrated & \textbf{\boldmath$0.059 \pm 0.004~(1)$} & 0.130 $\pm$ 0.002~(2)~\CIRCLE & 0.199 $\pm$ 0.003~(3)~\CIRCLE & 0.238 $\pm$ 0.052~(4)~\CIRCLE \\
Risk Parity Enhanced & \textbf{\boldmath$0.818 \pm 0.015~(1)$} & 0.862 $\pm$ 0.004~(2)~\CIRCLE & 0.899 $\pm$ 0.003~(4)~\CIRCLE & 0.864 $\pm$ 0.017~(3)~\CIRCLE \\
Balanced Growth & \textbf{\boldmath$0.672 \pm 0.010~(1)$} & 0.710 $\pm$ 0.004~(2)~\CIRCLE & 0.710 $\pm$ 0.004~(2)~\CIRCLE & 0.767 $\pm$ 0.080~(4)~\CIRCLE \\
\midrule
\midrule
\textbf{Avg. Rank} & \textbf{1.13} & 2.38 & 2.63 & 3.75 \\
\bottomrule
\end{tabular}%
}
\end{minipage}
\caption{\textbf{Averaged gap-weighted regret results across financial market scenarios.} Values report averaged cumulative regret over time horizon (mean $\pm$ std), with ranks shown in parentheses; lower is better. The best performance in each scenario is highlighted in bold. Statistical significance (t-test): \CIRCLE~indicates the baseline is significantly worse than PrefUCB ($p < 0.05$); \Circle~indicates no significant difference. PrefUCB achieves the best average rank (1.13) and wins in 7 out of 8 scenarios.}
\label{tab:main_results}
\vspace{-2mm}
\end{table*}

\textbf{Trends in sub-objective performance.} We also analyze performance trends at the sub-objective level. For clarity of presentation, the details are provided in Appendix~\ref{sub:temporal_dynamics_of_objectives}.

%% file: sections/related.tex

\section{Related Works}

\textbf{Multi-objective optimization and selective ensemble}.
Selective ensemble naturally involves trade-offs among predictive accuracy, diversity, and model complexity~\citep{zhou2012ensemble,DBLP:journals/ai/ZhouWT02}. Existing methods typically cast ensemble pruning/selection as an offline multi-objective optimization problem and approximate the Pareto frontier via evolutionary or pruning-based procedures~\citep{DBLP:conf/aaai/QianYZ15,DBLP:books/sp/ZhouYQ19}. More broadly, preference-guided multi-objective optimization biases the search toward desirable regions of the frontier through scalarization~\citep{DBLP:conf/uai/PariaKP19}, preference-aware constraints~\citep{DBLP:conf/nips/AbdolshahSR0V19,chen2024ferero}, or cone-based dominance relations~\citep{censor1977pareto,nasrabadi2019convex,korhonen2016dual}. In comparison, we study selective ensemble as a budgeted sequential learning problem over a fixed model pool, where multi-objective performance is revealed through stochastic evaluations on sampled data. Accordingly, the preference cone is used to shape online evaluation and selection, rather than to refine an accessible frontier. Selective ensemble and model selection have also been studied in active learning, where ensemble-based querying can jointly address data selection and model selection under limited labeling budgets~\citep{sugiyama2008batch,sugiyama2008active}.

\textbf{Multi-objective bandits with preferences}.
Classical multi-objective bandits study vector-valued rewards under the standard componentwise Pareto order, either through Pareto-regret minimization~\citep{DBLP:conf/ijcnn/DruganN13,xu2023pareto} or Pareto-front identification~\citep{auer2016pareto,crepon2024sequential}; contextual variants have also been studied~\citep{lu2019multi}. A separate line assumes fully specified preferences, including scalarized criteria such as the generalized Gini index~\citep{busa2017gini}, dominant-objective formulations~\citep{tekin2018dominant}, lexicographic orders~\citep{huyuk2021lex}, and hierarchical Pareto-lexicographic rules~\citep{cheng2024hierarchize}. Another line learns a complete utility model from user feedback, e.g., through pairwise comparisons~\citep{DBLP:conf/aldt/RoijersZN17,reymond2024interactive} or preference estimation~\citep{cao2026customization}. Preference-based feedback has also been studied in dueling bandits, including recent distributed formulations that address collaborative exploration from preference comparisons~\citep{raveh2025multiplayer}. Closest to our work, recent pure-exploration studies model incomplete preferences by a polyhedral cone~\citep{ararat2023vector,karagozlu2024learning} and analyze fixed-confidence identification under cone-induced optimality~\citep{shukla2024preference}. We instead study regret minimization under cone-induced partial preferences. The key distinction is that we must make high-quality decisions throughout learning, not only certify the correct set at stopping. To the best of our knowledge, this regret-minimizing formulation has not been studied before, and it subsumes both Pareto-based and scalarization-based bandit models as special cases.

%% file: sections/appendix.tex

\newpage
\appendix
\onecolumn
\begin{center}
\LARGE{\textbf{Appendix}}
\end{center}

\section{Proof of Theorem~\ref{thm:main} and Corollary~\ref{cor:gap-regret}}
\label{app:theorem}

In this section, we provide a complete proof of Theorem~\ref{thm:main}. We first establish several auxiliary lemmas and then derive an arm-wise pull-count bound for every suboptimal action. Theorem~\ref{thm:main} follows by summing these bounds over all suboptimal actions, and Corollary~\ref{cor:gap-regret} is then immediate.

\subsection{Auxiliary Lemmas}

To justify the use of KL-style confidence bounds on $\hat{S}_k(a)$, we note the following property.

\begin{myLemma}
\label{lem:transformed-feedback}
For each $k\in[M]$, define the scalar transformed feedback $Z_{t,k}(a) := S_k(\bm{y}_t(a))$. Since $\bm{y}_t(a)\in[0,1]^d$, by the construction of $L_k$ and $U_k$ in Definition~\ref{def:norm-proj}, we have $Z_{t,k}(a)\in[0,1]$ and $\mathbb{E}[Z_{t,k}(a)] = S_k(\bm{r}(a))$. Moreover, $\hat{S}_k(a) = \frac{1}{n(a)}\sum_{\tau: a_\tau = a} Z_{\tau,k}(a)$. Since the reward vectors $\{\bm{y}_t(a)\}_{t\geq 1}$ are independent and identically distributed for each arm $a$ (cf.\ Section~\ref{sec:problem-formulation}), $\hat{S}_k(a)$ is the empirical mean of $[0,1]$-bounded independent observations with common expectation $S_k(\bm{r}(a))$.
\end{myLemma}

Lemma~\ref{lem:transformed-feedback} ensures that for each preference direction $k$, the empirical projection $\hat{S}_k(a)$ is the sample mean of $[0,1]$-bounded random variables with expectation $S_k(\bm{r}(a))$. This allows us to apply standard KL-UCB-style concentration arguments for $[0,1]$-bounded random variables~\citep{garivier2011kl} to each directional projection.

For each action $a\in[K]$ and direction $k\in[M]$, write
\[
S_k(a) := S_k(\bm r(a))
\]
for the true normalized projection of action $a$ along direction $k$.

For the analysis, it is convenient to index transformed observations by sample order within each arm. Specifically, for each $(a,k)$, let $\{Z_{a,k,s}\}_{s\ge 1}$ denote the scalar sequence obtained by applying the projection $S_k(\cdot)$ to the successive reward observations of arm $a$:
\[
Z_{a,k,s} := S_k(\bm y_{a,s}),
\]
where $\bm y_{a,s}\in[0,1]^d$ is the $s$-th reward sample generated by arm $a$. By Lemma~\ref{lem:transformed-feedback} in the main text, for every fixed $(a,k)$, the sequence $\{Z_{a,k,s}\}_{s\ge 1}$ is i.i.d., takes values in $[0,1]$, and satisfies
\[
\mathbb E[Z_{a,k,s}] = S_k(a).
\]

For each round $t$, let $n_t(a)$ denote the number of times arm $a$ has been pulled \emph{before} round $t$, and define
\[
\hat S_k(a,t) := \frac{1}{n_t(a)}\sum_{s=1}^{n_t(a)} Z_{a,k,s}.
\]
Since PrefUCB pulls each arm once during initialization, we always have $n_t(a)\ge 1$ whenever the index is evaluated.

We also define the directional KL-UCB index
\[
\UCB_k(a,t)
:=
\sup\Bigl\{
q\in[\hat S_k(a,t),1]:
d(\hat S_k(a,t)\,\|\,q)\le \frac{\log t}{n_t(a)}
\Bigr\},
\]
where
\[
d(x\|y):= x\log\frac{x}{y} + (1-x)\log\frac{1-x}{1-y}
\]
is the binary Kullback--Leibler divergence. The PrefUCB index of arm $a$ at round $t$ is therefore
\[
I_t(a) := \min_{k\in[M]} \UCB_k(a,t).
\]

\begin{myLemma}[Chernoff--KL Concentration for Bounded Variables]
\label{lem:worst-case}
Let $X_1,\ldots,X_n$ be i.i.d.\ random variables taking values in $[0,1]$ with mean $p$. Then for any $\epsilon>0$,
\begin{align}
\Pr\!\left[\frac{1}{n}\sum_{i=1}^n X_i \ge p+\epsilon\right]
&\le \exp\!\left(-n\,d(p+\epsilon\,\|\,p)\right), \label{eq:chernoff-upper-app}\\
\Pr\!\left[\frac{1}{n}\sum_{i=1}^n X_i \le p-\epsilon\right]
&\le \exp\!\left(-n\,d(p-\epsilon\,\|\,p)\right), \label{eq:chernoff-lower-app}
\end{align}
where $d(\cdot\|\cdot)$ denotes the binary KL divergence. Moreover, for all $p,q\in[0,1]$,
\begin{equation}
\label{eq:pinsker-app}
d(q\|p)\ge 2(q-p)^2.
\end{equation}
\end{myLemma}

\begin{proof}
The tail bounds~\eqref{eq:chernoff-upper-app}--\eqref{eq:chernoff-lower-app} are standard Chernoff--KL inequalities for $[0,1]$-bounded random variables~\citep{garivier2011kl}. Inequality~\eqref{eq:pinsker-app} is the usual Pinsker-type lower bound for binary relative entropy.
\end{proof}

\begin{myLemma}[UCB Width Bound]
\label{lem:ucb-width}
For any action $a$, direction $k$, and round $t$,
\[
\UCB_k(a,t)\le \hat S_k(a,t) + \sqrt{\frac{\log t}{2\,n_t(a)}}.
\]
\end{myLemma}

\begin{proof}
Let $\hat p := \hat S_k(a,t)$ and let $q\ge \hat p$. By~\eqref{eq:pinsker-app},
\[
d(\hat p\|q)\ge 2(q-\hat p)^2.
\]
Hence, if
\[
q > \hat p + \sqrt{\frac{\log t}{2\,n_t(a)}},
\]
then
\[
d(\hat p\|q) > 2\cdot\frac{\log t}{2\,n_t(a)} = \frac{\log t}{n_t(a)},
\]
so such a $q$ cannot belong to the feasible set defining $\UCB_k(a,t)$. Therefore,
\[
\UCB_k(a,t)\le \hat S_k(a,t) + \sqrt{\frac{\log t}{2\,n_t(a)}}.
\]
\end{proof}

\begin{myLemma}[UCB Coverage]
\label{lem:ucb-coverage}
For any action $a$, direction $k$, and horizon $T$,
\[
\sum_{t=K+1}^{T}
\Pr\!\left[\UCB_k(a,t) < S_k(a)\right]
\le
C_{\mathrm{kl}}(1+\log T),
\]
where $C_{\mathrm{kl}}>0$ is a universal constant.
\end{myLemma}

\begin{proof}
Fix $(a,k)$ and write $p:=S_k(a)$. The observations
$\{Z_{a,k,s}\}_{s\ge 1}$ form an i.i.d.\ $[0,1]$-bounded stream with mean $p$, and the index $\UCB_k(a,t)$ is the usual KL-UCB index evaluated at the adaptively chosen sample count $n_t(a)$.
By the standard self-normalized coverage lemma for KL-UCB with bounded rewards~\citep{garivier2011kl}, for any adaptive sampling rule,
\[
\sum_{t=K+1}^{T}
\Pr\!\left[
\sup\Bigl\{q\in[\hat S_k(a,t),1]:
n_t(a)d(\hat S_k(a,t)\|q)\le \log t
\Bigr\}
<p
\right]
\le
C_{\mathrm{kl}}(1+\log T).
\]
We use this lemma only as a standard adaptive-sampling coverage result; the problem-specific step is the reduction of each preference direction to a bounded scalar feedback stream.
The event inside the probability is exactly $\{\UCB_k(a,t)<S_k(a)\}$ by definition, which proves the claim.
\end{proof}

\begin{myLemma}[Selection Elimination]
\label{lem:selection}
Fix any suboptimal action $a\notin P_C^*$. Let $a^\star\in\mathcal D(a)$ be a dominating Pareto-$C$-optimal action achieving the preference gap $\Delta_a$, namely,
\[
\Delta_a
=
\min_{k\in[M]}\bigl(S_k(a^\star)-S_k(a)\bigr).
\]
Define
\[
m_0 := \left\lceil \frac{2\log T}{\Delta_a^2}\right\rceil.
\]
Suppose that at round $t$ the following two events hold:
\begin{enumerate}
\item[(G1)] For all $k\in[M]$,
\[
\hat S_k(a,t) < S_k(a) + \frac{\Delta_a}{2};
\]
\item[(G2)] For all $k\in[M]$,
\[
\UCB_k(a^\star,t)\ge S_k(a^\star).
\]
\end{enumerate}
If, in addition, $n_t(a)\ge m_0$, then action $a$ cannot be selected by PrefUCB at round $t$.
\end{myLemma}

\begin{proof}
Assume that (G1) holds and that $n_t(a)\ge m_0$. By Lemma~\ref{lem:ucb-width}, for every $k\in[M]$,
\[
\UCB_k(a,t)\le \hat S_k(a,t) + \sqrt{\frac{\log t}{2\,n_t(a)}}.
\]
Since $t\le T$ and $n_t(a)\ge m_0$,
\[
\sqrt{\frac{\log t}{2\,n_t(a)}}
\le
\sqrt{\frac{\log T}{2\,m_0}}
\le
\frac{\Delta_a}{2}.
\]
Combining these inequalities with (G1) yields, for every $k\in[M]$,
\[
\UCB_k(a,t)
<
S_k(a) + \frac{\Delta_a}{2} + \frac{\Delta_a}{2}
=
S_k(a)+\Delta_a.
\]
Because $a^\star$ achieves the gap $\Delta_a$,
\[
\Delta_a
=
\min_{k\in[M]}\bigl(S_k(a^\star)-S_k(a)\bigr),
\]
and therefore, for every $k\in[M]$,
\[
S_k(a^\star)-S_k(a)\ge \Delta_a,
\qquad\text{equivalently}\qquad
S_k(a)+\Delta_a \le S_k(a^\star).
\]
Hence, for all $k\in[M]$,
\[
\UCB_k(a,t) < S_k(a^\star).
\]
This is a componentwise comparison before applying the worst-direction aggregation. Taking the minimum over $k$ gives
\[
I_t(a)=\min_{k\in[M]}\UCB_k(a,t)
<
\min_{k\in[M]} S_k(a^\star).
\]
On the other hand, under (G2),
\[
I_t(a^\star)
=
\min_{k\in[M]}\UCB_k(a^\star,t)
\ge
\min_{k\in[M]}S_k(a^\star).
\]
Thus the scalar comparison between $I_t(a)$ and $I_t(a^\star)$ follows from componentwise dominance across all preference directions.
Consequently,
\[
I_t(a)<I_t(a^\star),
\]
which implies that PrefUCB cannot select arm $a$ at round $t$.
\end{proof}

\subsection{Main Proof}

\begin{proof}[Proof of Theorem~\ref{thm:main}]
If $T\le K$, the algorithm performs only the initialization phase, so each arm is pulled at most once and the theorem is immediate. We therefore consider the case $T>K$.

Fix any suboptimal action $a\notin P_C^*$, and let $a^\star\in\mathcal D(a)$ and $m_0$ be as in Lemma~\ref{lem:selection}. We begin by bounding the expected number of pulls of arm $a$.

Since arm $a$ is pulled once during initialization, the number of rounds in which arm $a$ can be selected while its pre-round pull count is still smaller than $m_0$ is at most $m_0$. Therefore,
\begin{equation}
\label{eq:pull-decomp-app}
N_a(T)
\le
m_0 + \sum_{t=K+1}^{T} \mathbb{I}\!\left(a_t=a,\ n_t(a)\ge m_0\right).
\end{equation}
Taking expectations gives
\begin{equation}
\label{eq:pull-decomp-exp-app}
\mathbb E[N_a(T)]
\le
m_0 + \sum_{t=K+1}^{T}\Pr\!\left(a_t=a,\ n_t(a)\ge m_0\right).
\end{equation}

By Lemma~\ref{lem:selection}, on the event $\{a_t=a,\ n_t(a)\ge m_0\}$ at least one of the good events (G1) or (G2) must fail. Define the corresponding bad events
\begin{align*}
B_1(t)
&:=
\Bigl\{\exists\,k\in[M]:
\hat S_k(a,t)\ge S_k(a)+\tfrac{\Delta_a}{2}\Bigr\},\\
B_2(t)
&:=
\Bigl\{\exists\,k\in[M]:
\UCB_k(a^\star,t)<S_k(a^\star)\Bigr\}.
\end{align*}
Then
\[
\{a_t=a,\ n_t(a)\ge m_0\}
\subseteq
B_1(t)\cup B_2(t).
\]
Substituting this into~\eqref{eq:pull-decomp-exp-app}, we obtain
\begin{equation}
\label{eq:union-split-app}
\mathbb E[N_a(T)]
\le
m_0
+
\sum_{t=K+1}^{T}\Pr\!\left(a_t=a,\ n_t(a)\ge m_0,\ B_1(t)\right)
+
\sum_{t=K+1}^{T}\Pr\!\left(B_2(t)\right).
\end{equation}

\medskip
\noindent\textbf{Step 1: Bounding the $B_1$ contribution.}
For $m\ge 1$, write
\[
\bar S_{k,m}(a)
:=
\frac{1}{m}\sum_{s=1}^{m} Z_{a,k,s}.
\]
If arm $a$ is selected at a round with pre-round count $n_t(a)=m$, then $\hat S_k(a,t)=\bar S_{k,m}(a)$. Moreover, for each fixed $m$, arm $a$ can be selected with pre-round count $m$ at most once. Therefore,
\[
\sum_{t=K+1}^{T}
\mathbb{I}\!\left(a_t=a,\ n_t(a)\ge m_0,\ B_1(t)\right)
\le
\sum_{m=m_0}^{T}
\mathbb{I}\!\left\{
\exists\,k\in[M]:
\bar S_{k,m}(a)\ge S_k(a)+\frac{\Delta_a}{2}
\right\}.
\]
Taking expectations and applying a union bound over directions gives
\begin{align}
\sum_{t=K+1}^{T}
\Pr\!\left(a_t=a,\ n_t(a)\ge m_0,\ B_1(t)\right)
&\le
\sum_{m=m_0}^{T}
\sum_{k=1}^{M}
\Pr\!\left(
\bar S_{k,m}(a)\ge S_k(a)+\frac{\Delta_a}{2}
\right) \notag\\
&\le
M\sum_{m=m_0}^{T}
\exp\!\left(-\frac{m\Delta_a^2}{2}\right).
\label{eq:B1-sample-count-app}
\end{align}
Here $S_k(a)+\Delta_a/2\le 1$ for every $k$, because the choice of $a^\star$ gives $S_k(a)+\Delta_a\le S_k(a^\star)\le 1$. The last step uses the Chernoff--KL bound~\eqref{eq:chernoff-upper-app} and Pinsker's inequality~\eqref{eq:pinsker-app}, since
\[
d\!\left(S_k(a)+\tfrac{\Delta_a}{2}\,\middle\|\,S_k(a)\right)
\ge
\frac{\Delta_a^2}{2}
\]
for every $k\in[M]$.
Since $\Delta_a\in(0,1]$, we have
$1-\exp(-\Delta_a^2/2)\ge \Delta_a^2/4$. Therefore,
\[
\sum_{m=m_0}^{T}
\exp\!\left(-\frac{m\Delta_a^2}{2}\right)
\le
\frac{\exp(-m_0\Delta_a^2/2)}
{1-\exp(-\Delta_a^2/2)}
\le
\frac{4}{T\Delta_a^2},
\]
where the last step uses $m_0\ge 2\log T/\Delta_a^2$. Thus,
\begin{equation}
\label{eq:B1-total-app}
\sum_{t=K+1}^{T}
\Pr\!\left(a_t=a,\ n_t(a)\ge m_0,\ B_1(t)\right)
\le
\frac{4M}{T\Delta_a^2}
\le
\frac{4M}{\Delta_a^2}.
\end{equation}

\medskip
\noindent\textbf{Step 2: Bounding the $B_2$ contribution.}
Lemma~\ref{lem:ucb-coverage} and a union bound over directions imply
\begin{equation}
\label{eq:B2-total-app}
\sum_{t=K+1}^{T}\Pr\!\left(B_2(t)\right)
\le
M C_{\mathrm{kl}}(1+\log T).
\end{equation}

\medskip
\noindent\textbf{Step 3: Combining the bounds.}
Combining~\eqref{eq:union-split-app}, \eqref{eq:B1-total-app}, and \eqref{eq:B2-total-app}, we obtain
\[
\mathbb E[N_a(T)]
\le
m_0 + \frac{4M}{\Delta_a^2} + M C_{\mathrm{kl}}(1+\log T).
\]
Since
\[
m_0=\left\lceil \frac{2\log T}{\Delta_a^2}\right\rceil
\le
\frac{2\log T}{\Delta_a^2}+1,
\]
and since $\Delta_a\in(0,1]$ and $\log T\ge 1$ for $T\ge 3$, the last display implies that there exists a universal constant $C_0>0$ such that
\begin{equation}
\label{eq:na-final-app}
\mathbb E[N_a(T)]
\le
C_0\,\frac{(M+1)\log T}{\Delta_a^2}+1.
\end{equation}
For the remaining finite cases $K<T<3$, the bound~\eqref{eq:na-final-app} is trivially valid after increasing $C_0$ if necessary.

Finally, using
\[
\mathcal R_T^{01}
=
\sum_{a\notin P_C^*} N_a(T),
\]
taking expectations and summing~\eqref{eq:na-final-app} over all suboptimal actions yields
\[
\mathbb E[\mathcal R_T^{01}]
\le
\sum_{a\notin P_C^*}
\left(
C_0\,\frac{(M+1)\log T}{\Delta_a^2}+1
\right),
\]
which proves Theorem~\ref{thm:main}.
\end{proof}

\begin{proof}[Proof of Corollary~\ref{cor:gap-regret}]
By definition,
\[
\mathcal R_T^{\mathrm{gap}}
=
\sum_{a\notin P_C^*}\Delta_a\,N_a(T).
\]
Taking expectations and applying the arm-wise pull-count bound~\eqref{eq:na-final-app}, we obtain
\begin{align*}
\mathbb E[\mathcal R_T^{\mathrm{gap}}]
&=
\sum_{a\notin P_C^*}\Delta_a\,\mathbb E[N_a(T)]\\
&\le
\sum_{a\notin P_C^*}
\Delta_a
\left(
C_0\,\frac{(M+1)\log T}{\Delta_a^2}+1
\right)\\
&=
\sum_{a\notin P_C^*}
\left(
C_0\,\frac{(M+1)\log T}{\Delta_a}
+\Delta_a
\right),
\end{align*}
which proves Corollary~\ref{cor:gap-regret}.
\end{proof}

The proof above is intentionally conservative in order to match the theorem statement in the main text. The linear dependence on $M$ arises from union bounds over the $M$ preference directions, reflecting the fact that PrefUCB evaluates each arm through the worst-direction score $\min_{k\in[M]}\UCB_k(a,t)$. Sharper leading constants are possible with a more refined treatment of the coverage events, but such refinements are not needed for the stated result.

\section{Experimental Details and Additional Results}
\label{app:mandate_details}

This appendix provides a comprehensive description of the preference construction methodology, the specific logic defining each financial mandate, and detailed performance visualizations.

\begin{table*}[!h]
\centering
\resizebox{.95\textwidth}{!}{
\begin{tabularx}{\textwidth}{@{}c|>{\centering\arraybackslash}X|c@{}}
\toprule
\hline
\textbf{Category} & \textbf{Scenario and Description} & \textbf{Preference Matrix} \(\mathbf{W}\) \\
\hline
\multirow{3}{*}{\textbf{Standard Benchmarks}} 
& \textbf{Standard}: Standard Pareto optimality benchmark & \multirow{-1}{*}{\(\mathbf{I}_7\)} \\
& \textbf{Scalarization}: Simple weighted combination & \multirow{-1}{*}{\([.2,.1,.2,.2,.1,.1,.1]\)} \\
& \textbf{Single-objective}: Classical MAB reduction & \multirow{-1}{*}{\([1,0,\dots,0]\)} \\
\hline
\multirow{12}{*}{\textbf{Real-world Applications}} 
& \textbf{Customer Service}: Major preference on multitasking, maintain acceptable accuracy and reasoning, sacrifice advanced problem-solving for more multitasking ability & \multirow{3}{*}{\(\begin{bmatrix} 1 & 0 & 1 & 0 & 0 & 2 & 1 \\ 0 & -1 & 0 & 0 & 0 & 1 & 0 \end{bmatrix}\)} \\
& \textbf{Legal Document Review}: Major preference on reasoning, maintain acceptable advanced problem-solving and domain knowledge, sacrifice multitasking for more reasoning & \multirow{3}{*}{\(\begin{bmatrix} 0 & 1 & 2 & 1 & 0 & 0 & 0 \\ 0 & 0 & 1 & 0 & 0 & -1 & 0 \end{bmatrix}\)} \\
& \textbf{High-frequency Trade}: Major preference on processing speed, maintain acceptable domain-knowledge, sacrifice advanced problem-solving for more speed & \multirow{3}{*}{\(\begin{bmatrix} 0 & 0 & 0 & 1 & 0 & 0 & 2 \\ 0 & -1 & 0 & 0 & 0 & 0 & 1 \end{bmatrix}\)} \\
& \textbf{Scientific Discovery}: Major preference on advanced problem-solving, maintain acceptable reasoning, domain knowledge and consistency, sacrifice speed for more advanced problem-solving ability & \multirow{4}{*}{\(\begin{bmatrix} 0 & 2 & 1 & 1 & 1 & 0 & 0 \\ 0 & 1 & 0 & 0 & 0 & 0 & -1 \end{bmatrix}\)} \\
\hline
\bottomrule
\end{tabularx}
}
\caption{Preference direction matrices for various scenarios. The sub-objectives considered are accuracy, advanced problem-solving capability, reasoning, domain knowledge, interaction consistency, multitasking ability, and processing speed. Higher numerical values indicate stronger preferences, while negative values signify sub-objectives that can be relaxed.}
\label{tab:preferences}
\end{table*}

\begin{table*}[!h]
\centering
\caption{Financial Mandates and Preference Logic. The weight vectors ($W$) are constructed to enforce specific trade-offs, such as penalizing volatility (negative weights) while maximizing Alpha (positive weights).}
\label{tab:mandate_descriptions}
\small
\resizebox{.95\textwidth}{!}{
\begin{tabular}{l p{5cm} p{8cm}}
\toprule
\textbf{Mandate} & \textbf{Theoretical Goal} & \textbf{Preference Construction Logic (Excerpt)} \\
\midrule

\textbf{Aggressive Growth} & Absolute Return + Alpha + Risk Tolerance & Maximize CWG ($w \approx 0.63$) and Alpha ($w \approx 0.37$). Negative penalties applied to MDD and TO are relaxed or removed to allow aggressive positioning. \\
\midrule

\textbf{Market Neutral} & Alpha + Calmar + Downside Protection & Primary focus on Alpha ($w=1.0$) and CR. Strong negative penalties on correlation and drawdown to ensure neutrality. \\
\midrule

\textbf{High Frequency Trading} & Sharpe + Sortino + High Turnover Tolerance & Balanced emphasis on SR ($w \approx 0.50$) and Sortino ($w \approx 0.50$). Turnover penalties are removed to accommodate high-frequency rebalancing. \\
\midrule

\textbf{Income Focused} & Consistent Yield + Low Volatility & Balanced weights on CWG ($\approx 0.60$) and SR ($\approx 0.40$). Negative weights on Turnover and Alpha to discourage speculative trading. \\
\midrule

\textbf{Low Volatility} & Double Risk Control (MDD + Sortino) & Pure risk minimization: Weights distributed between MDD ($w \approx 0.72$) and Sortino ($w \approx 0.28$). Strategies with high volatility receive severe negative scores. \\
\midrule

\textbf{Momentum Conc.} & Alpha + CWG Maximization & Similar to Aggressive Growth but with higher specific weight on Alpha ($\approx 0.63$) to capture trend-following premiums. \\
\midrule

\textbf{Risk Parity Enh.} & SR + MDD Balance & Weights distributed between SR ($\approx 0.29$) and MDD ($\approx 0.71$) to enforce risk-adjusted returns over pure growth. \\
\midrule

\textbf{Balanced Growth} & Multi-Objective Equilibrium & A complex mix of CWG, SR, and Sortino, ensuring no single metric dominates. Penalties are moderate. \\

\bottomrule
\end{tabular}
}
\end{table*}

\begin{table*}[!h]
\caption{Detailed Weight Specifications for Financial Mandates. Each mandate is represented by three weight vectors $w_1, w_2, w_3 \in \mathbb{R}^7$, corresponding to the seven objectives: CWG, SR, MDD, Sortino, Alpha, TO, CR.}
\label{tab:weight_specifications}
\centering
\footnotesize
\begin{tabular}{lcccccccc}
\toprule
\textbf{Mandate} & \textbf{Vector} & \textbf{CWG} & \textbf{SR} & \textbf{MDD} & \textbf{Sortino} & \textbf{Alpha} & \textbf{TO} & \textbf{CR} \\
\midrule
\multirow{3}{*}{\textbf{Aggressive Growth}}
& $w_1$ & 0.633 & 0.000 & -0.300 & 0.000 & 0.367 & 0.000 & 0.000 \\
& $w_2$ & 1.000 & 0.000 & 0.000 & -0.300 & 0.000 & -0.300 & 0.000 \\
& $w_3$ & 0.000 & 0.000 & -0.300 & 0.000 & 0.475 & -0.300 & 0.525 \\
\midrule
\multirow{3}{*}{\textbf{Market Neutral Hedge}}
& $w_1$ & 0.000 & 0.000 & 0.000 & 0.000 & 1.000 & 0.000 & 0.000 \\
& $w_2$ & -0.300 & 0.000 & 0.000 & 0.000 & 0.475 & 0.000 & 0.525 \\
& $w_3$ & 0.000 & 0.000 & 0.716 & 0.284 & 0.000 & -0.300 & 0.000 \\
\midrule
\multirow{3}{*}{\textbf{High Frequency Trading}}
& $w_1$ & 0.000 & 0.505 & 0.000 & 0.495 & 0.000 & 0.000 & 0.000 \\
& $w_2$ & 0.000 & 1.000 & -0.300 & 0.000 & -0.300 & 0.000 & 0.000 \\
& $w_3$ & 0.000 & 0.000 & 0.000 & 0.510 & 0.000 & -0.300 & 0.490 \\
\midrule
\multirow{3}{*}{\textbf{Income Focused}}
& $w_1$ & 0.595 & 0.405 & -0.300 & 0.000 & 0.000 & 0.000 & 0.000 \\
& $w_2$ & 0.000 & 0.000 & 0.000 & 1.000 & -0.300 & -0.300 & 0.000 \\
& $w_3$ & 0.000 & 0.515 & 0.000 & 0.000 & 0.000 & -0.300 & 0.485 \\
\midrule
\multirow{3}{*}{\textbf{Low Volatility}}
& $w_1$ & 0.000 & 0.000 & 0.716 & 0.284 & 0.000 & 0.000 & 0.000 \\
& $w_2$ & -0.300 & 0.000 & 0.000 & 1.000 & -0.300 & -0.300 & 0.000 \\
& $w_3$ & 0.000 & 0.287 & 0.713 & 0.000 & 0.000 & 0.000 & -0.300 \\
\midrule
\multirow{3}{*}{\textbf{Momentum Concentrated}}
& $w_1$ & 0.633 & 0.000 & 0.000 & 0.000 & 0.367 & 0.000 & 0.000 \\
& $w_2$ & 0.000 & 0.000 & -0.300 & -0.300 & 1.000 & -0.300 & 0.000 \\
& $w_3$ & 0.609 & -0.300 & 0.000 & 0.000 & 0.000 & -0.300 & 0.391 \\
\midrule
\multirow{3}{*}{\textbf{Risk Parity Enhanced}}
& $w_1$ & 0.000 & 0.287 & 0.713 & 0.000 & 0.000 & 0.000 & 0.000 \\
& $w_2$ & 0.000 & 1.000 & 0.000 & 0.000 & 0.000 & -0.300 & 0.000 \\
& $w_3$ & -0.300 & 0.000 & 0.716 & 0.284 & -0.300 & 0.000 & 0.000 \\
\midrule
\multirow{3}{*}{\textbf{Balanced Growth}}
& $w_1$ & 0.426 & 0.290 & 0.000 & 0.284 & 0.000 & 0.000 & 0.000 \\
& $w_2$ & 0.000 & 0.515 & -0.300 & 0.000 & 0.000 & -0.300 & 0.485 \\
& $w_3$ & 0.000 & 0.000 & 0.000 & 0.536 & 0.464 & -0.300 & 0.000 \\
\bottomrule
\end{tabular}
\end{table*}

\begin{figure}[!t]
\centering

\begin{subfigure}{0.95\textwidth}
  \centering
  \includegraphics[width=\linewidth]{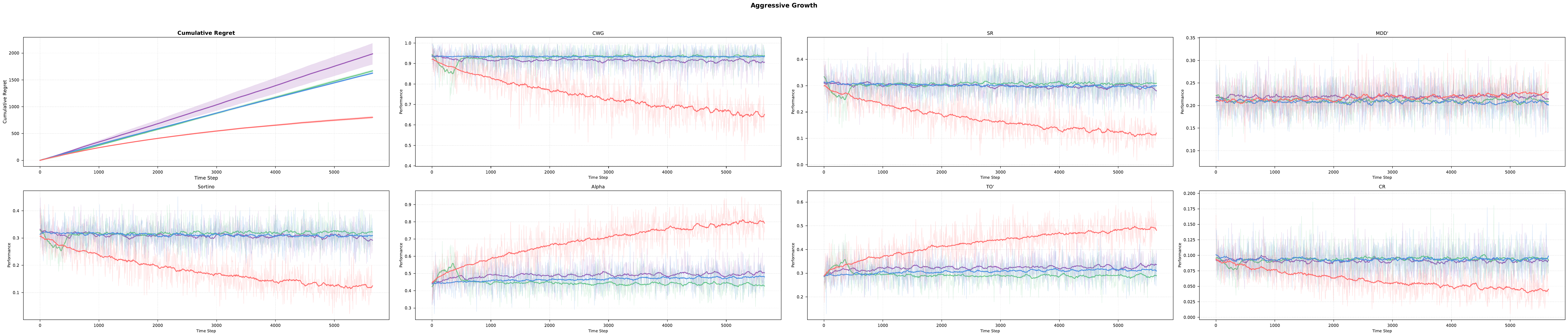}
  \caption{Aggressive Growth}
\end{subfigure}
\hfill
\begin{subfigure}{0.95\textwidth}
  \centering
  \includegraphics[width=\linewidth]{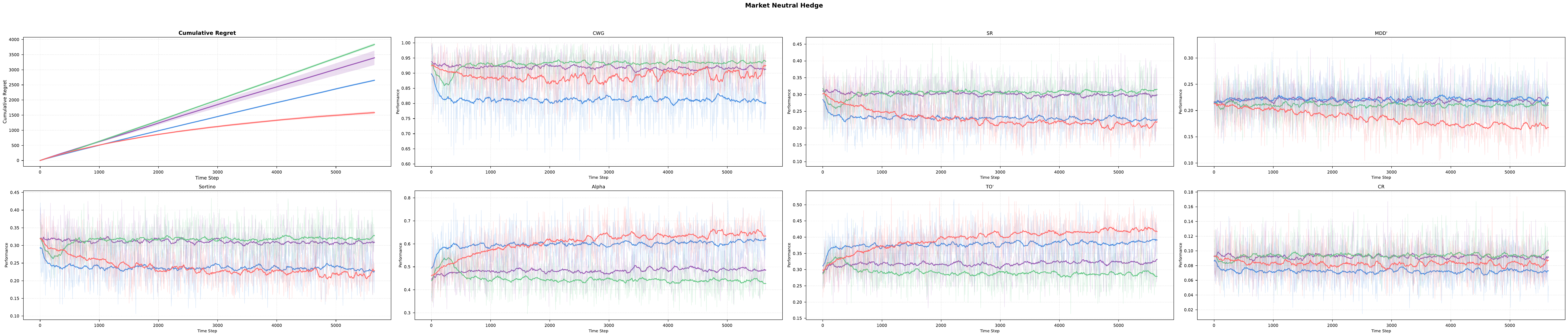}
  \caption{Market Neutral Hedge}
\end{subfigure}

\begin{subfigure}{0.95\textwidth}
  \centering
  \includegraphics[width=\linewidth]{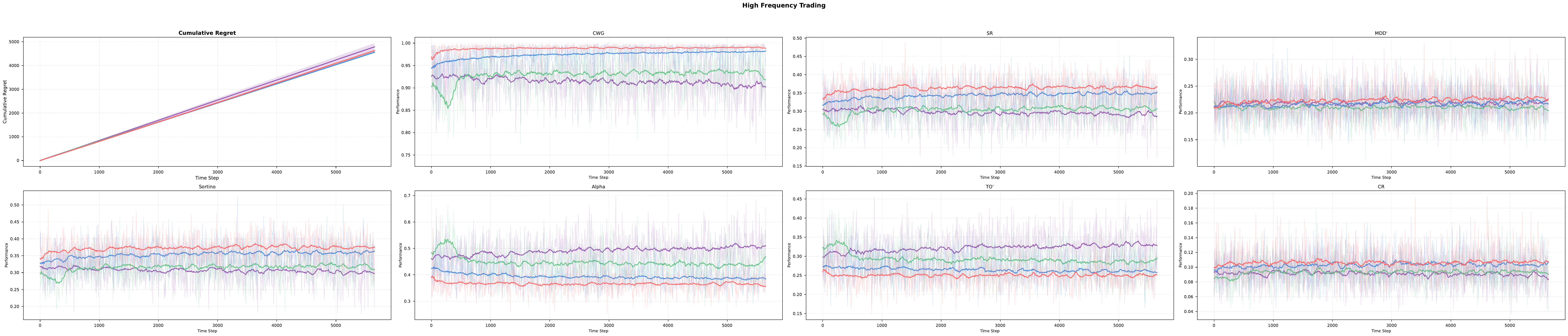}
  \caption{High Frequency Trading}
\end{subfigure}
\hfill
\begin{subfigure}{0.95\textwidth}
  \centering
  \includegraphics[width=\linewidth]{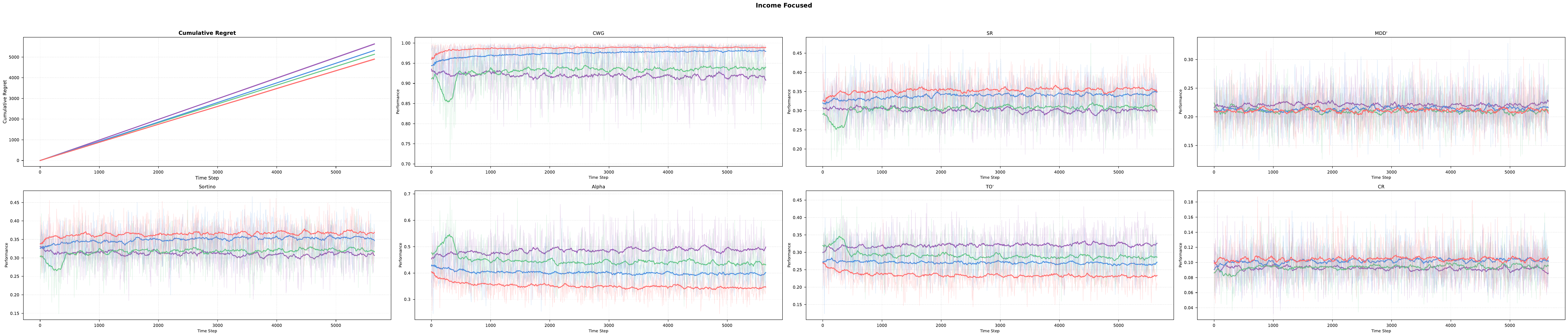}
  \caption{Income Focused performance}
\end{subfigure}
\begin{subfigure}{0.75\textwidth}
  \centering
  \includegraphics[width=\linewidth]{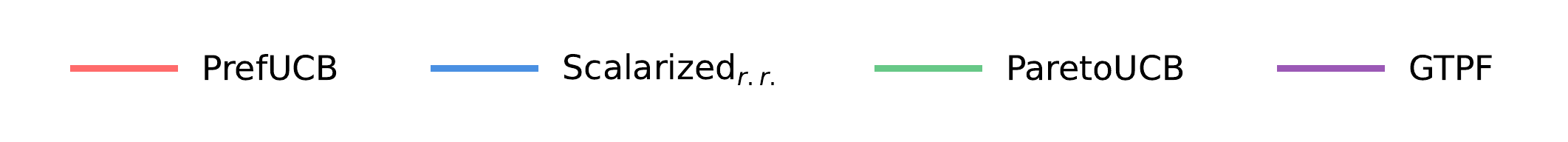}
\end{subfigure}

\caption{\textbf{Temporal trends of sub-objectives across different mandate scenarios.}}
\label{fig:sub-ops-a}
\end{figure}

\begin{figure}[!t]
\centering
\begin{subfigure}{0.95\textwidth}
  \centering
  \includegraphics[width=\linewidth]{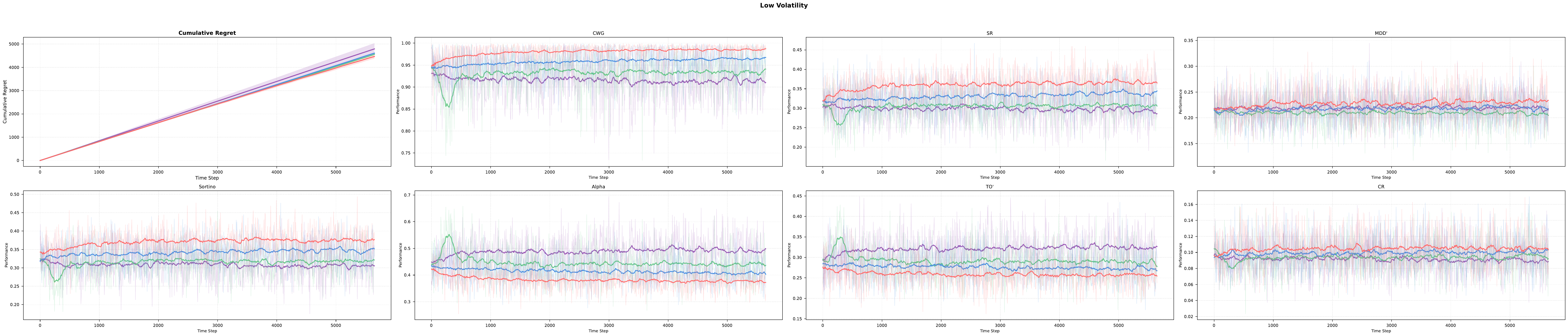}
  \caption{Low Volatility}
\end{subfigure}
\begin{subfigure}{0.95\textwidth}
  \centering
  \includegraphics[width=\linewidth]{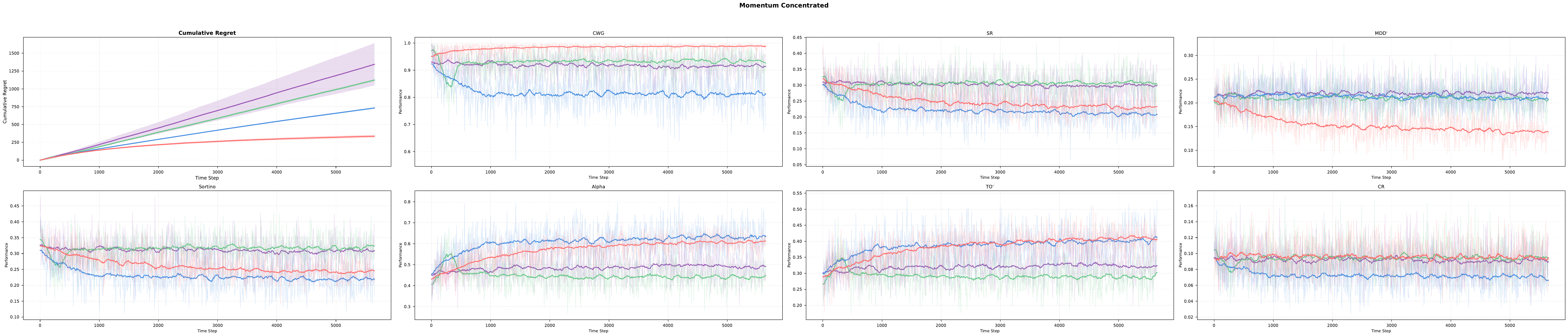}
  \caption{Momentum Concentrated}
\end{subfigure}
\begin{subfigure}{0.95\textwidth}
  \centering
  \includegraphics[width=\linewidth]{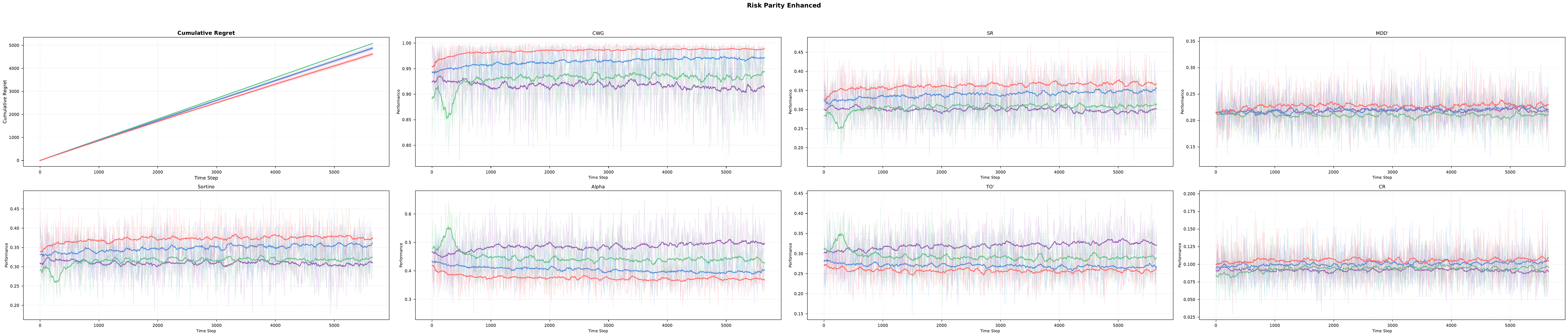}
  \caption{Risk Parity Enhanced}
\end{subfigure}
\begin{subfigure}{0.95\textwidth}
  \centering
  \includegraphics[width=\linewidth]{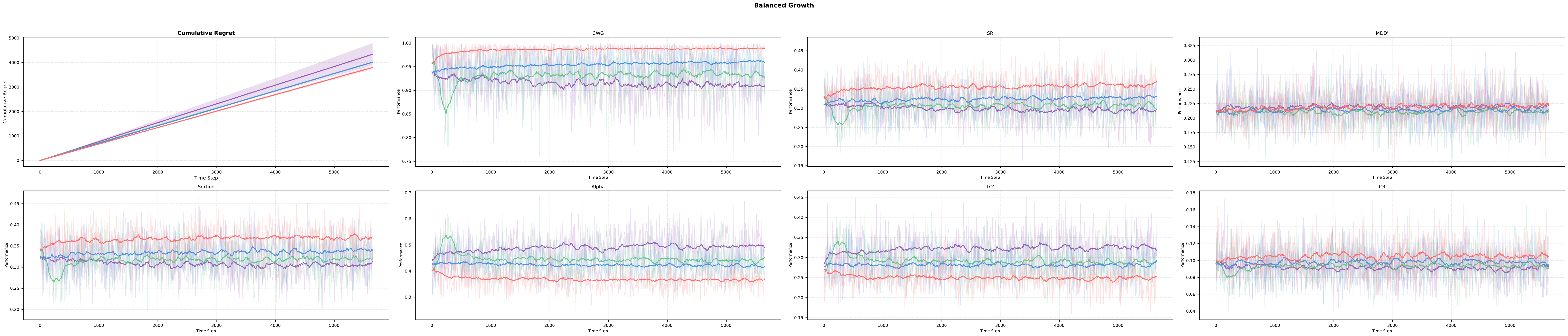}
  \caption{Balanced Growth}
\end{subfigure}
\begin{subfigure}{0.75\textwidth}
  \centering
  \includegraphics[width=\linewidth]{./fig/shared_algorithm_legend.pdf}
\end{subfigure}

\caption{\textbf{Temporal trends of sub-objectives across different mandate scenarios.}}
\label{fig:sub-ops-b}
\end{figure}

\subsection{Contenders and Implementation Details}
\label{app:contenders}

We compare the proposed PrefUCB algorithm with:
\begin{compactitem}[$\bullet$]
\item \textbf{ParetoUCB}~\citep{DBLP:conf/ijcnn/DruganN13}. As there are currently no algorithms explicitly designed for minimizing preference-directed regret, we adapt this state-of-the-art multi-objective bandit algorithms into selective ensemble methods, which selects base models according to standard Pareto optimality definition, without considering the preference direction;
\item \textbf{ScalarizedUCB\(_{r.r.}\)}~\citep{DBLP:conf/ijcnn/DruganN13}, which first converts multi-objective feedback into scalar rewards by randomly selecting a row from the preference direction matrix as the scalarization vector, which is taken as input of the selection algorithm;
\item \textbf{GT-PF} (\underline{g}round \underline{t}ruth \underline{P}areto \underline{f}ront). Given that existing selective ensemble methods lack mechanisms for incorporating user-specific preferences, GT-PF serves as an oracle-based benchmark that conducts random selection from the standard Pareto-optimal set, assuming complete knowledge of the true Pareto-optimal set, and thus providing the best possible performance for the preference-agnostic multi-objective selective ensemble methods~\citep{DBLP:conf/aaai/QianYZ15,DBLP:conf/ppsn/WuHQZ22} under perfect information assumptions.
\end{compactitem}

\subsection{Preference Direction Details}
\label{app:preference-direction-matrices}
Here, we show the preference direction details of some potential scenarios for LLM in the following Table~\ref{tab:preferences}.

\subsection{Data-Driven Preference Construction}

To evaluate our approach in realistic environments, we employ a rigorous data-driven heuristic to construct the preference matrix $W$ for each mandate, utilizing the historical performance of base strategies on the training partition.

\textbf{Construction Methodology.} The preference matrix $W$ is derived through a three-step procedure:
\begin{enumerate}
    \item \textbf{Metric Normalization:} All seven objective metrics are normalized to the unit interval $[0, 1]$. Objectives requiring minimization (MDD, TO) are inverted to create a unified maximization framework.
    \item \textbf{Variance-Based Priority:} For primary objectives, we assign positive weights proportional to their historical variance. This emphasizes metrics with high distinguishability.
    \item \textbf{Regression-Based Penalties:} We model trade-offs between conflicting objectives using linear regression. Negative weights are assigned to secondary objectives based on the regression coefficients ($\beta$) to penalize strategies that violate the mandate's risk tolerance.
\end{enumerate}

Each mandate is defined by a preference matrix $W \in \mathbb{R}^{3 \times 7}$, where each row $w_i$ represents an alternative weighting scheme. The algorithm evaluates strategies against all three weight vectors to capture the multi-faceted nature of institutional mandates.

\subsection{Mandate Descriptions and Weight Specifications}

We constructed $W$ for eight specific mandates. Table~\ref{tab:mandate_descriptions} details the theoretical goals and construction logic for each mandate, while Table~\ref{tab:weight_specifications} provides the exact weight vectors used in the experiments. Positive values indicate objectives to maximize, whereas negative values denote penalty terms for undesirable characteristics.

\subsection{Temporal Dynamics of Objectives}
\label{sub:temporal_dynamics_of_objectives}
Figures~\ref{fig:sub-ops-a} and~\ref{fig:sub-ops-b} illustrate the temporal trends of sub-objectives across different mandates. The solid lines represent the 100-round moving average, while the shaded regions correspond to the true reward values. The results confirm that \texttt{PrefUCB} (shown in red) consistently prioritizes preferred objectives, compromises on relaxable ones, and maintains acceptable levels for moderate objectives, demonstrating superior trade-off capabilities.